\documentclass{article}

\usepackage{microtype}
\usepackage{graphicx}
\usepackage{subcaption}
\usepackage{booktabs} 
\usepackage{multirow}

\PassOptionsToPackage{hyphens}{url}
\usepackage{hyperref}

\usepackage[preprint]{icml2026}

\usepackage{amsmath}
\usepackage{amssymb}
\usepackage{mathtools}
\usepackage{amsthm}

\usepackage{bbm}
\usepackage{amsfonts}
\usepackage{mathtools}
\usepackage{appendix}
\usepackage{xcolor}

\usepackage[capitalize,noabbrev]{cleveref}

\theoremstyle{plain}
\newtheorem{theorem}{Theorem}[section]
\newtheorem{proposition}[theorem]{Proposition}

\theoremstyle{definition}

\theoremstyle{remark}

\usepackage[textsize=tiny]{todonotes}

\def\dataset{\ensuremath{\mathcal{D}}}
\def\x{{\mathbf{x}}}

\def\b{{\mathbf{b}}}

\def\o{{\mathbf{o}}}
\def\z{{\mathbf{z}}}

\def\real{\mathbb{R}}

\def\dataset{\ensuremath{\mathcal{D}}}
\def\localdataset{\dataset^c}

\def\clients{\ensuremath{\mathcal{C}}}
\def\model{\theta}

\def\globallossfunction{{\loss}}

\def\W{\mathbf{W}}

\def\g{\mathbf{g}}

\def\Wout{\W^{(L)}}

\def\loss{\ensuremath{\mathcal{L}}}

\def\activeset{\mathcal{A}}
\def\gradlosswiatt{\frac{\partial \loss}{\partial \W_i^a}}

\def\slice{\mathbf{s}}

\DeclareMathOperator*{\argmin}{argmin}

\icmltitlerunning{No More Guessing}

\begin{document}

\twocolumn[
  \icmltitle{No More Guessing: a Verifiable Gradient Inversion Attack in Federated Learning}



  \icmlsetsymbol{equal}{*}

  \begin{icmlauthorlist}
    \icmlauthor{Francesco Diana}{uni,inria}
    \icmlauthor{Chuan Xu}{uni,inria,cnrs,i3s}
    \icmlauthor{André Nusser}{uni,inria,cnrs,i3s}
    \icmlauthor{Giovanni Neglia}{uni,inria}
  \end{icmlauthorlist}

  \icmlaffiliation{uni}{Université Côte d'Azur, France}
  \icmlaffiliation{inria}{Inria, France}
  \icmlaffiliation{cnrs}{CNRS, France}
  \icmlaffiliation{i3s}{I3S, France}

  \icmlcorrespondingauthor{Francesco Diana}{francesco.diana@inria.fr}

  \icmlkeywords{Machine Learning, ICML}

  \vskip 0.3in
]



\printAffiliationsAndNotice{}  

\begin{abstract}

Gradient inversion attacks threaten client privacy in federated learning by reconstructing training samples from clients’ shared gradients. Gradients aggregate contributions from multiple records and existing attacks may fail to disentangle them, yielding incorrect reconstructions with no intrinsic way to certify success. In vision and language, attackers may fall back on human inspection to judge reconstruction plausibility, but this is far less feasible for numerical tabular records , fueling the impression that tabular data is less vulnerable.

We challenge this perception by proposing a verifiable gradient inversion attack (VGIA) that provides an explicit certificate of correctness for reconstructed samples. Our method adopts a geometric view of ReLU leakage: the activation boundary of a fully connected layer defines a hyperplane in input space. VGIA introduces an algebraic, subspace-based verification test that detects when a hyperplane-delimited region contains exactly one record. Once isolation is certified, VGIA recovers the corresponding feature vector analytically and reconstructs the target via a lightweight optimization step.

Experiments on tabular benchmarks with large batch sizes demonstrate exact record and target recovery in regimes where existing state-of-the-art attacks either fail or cannot assess reconstruction fidelity. Compared to prior geometric approaches, VGIA allocates hyperplane queries more effectively, yielding faster reconstructions with fewer attack rounds.



\end{abstract}

\section{Introduction}
Federated Learning (FL)~\citep{mcmahan2017communication} is a framework that enables distributed training of machine learning models without transmission of private data from clients to a central server. By design, FL limits communication to local updates, such as gradients or parameters, under the assumption that keeping the data on the client guarantees privacy. However, recent research has fundamentally challenged this supposition, demonstrating that shared gradients retain significant imprint of the training data, allowing adversaries to reconstruct private inputs through Gradient Inversion Attacks (GIA) \citep{dlg, geiping_gi, kariyappa23a_cocktail, bakarsky2025spear++}.  


Early research primarily operated under the \textit{honest-but-curious} threat model, ignoring the risks posed by an untrusted third-party server or compromised infrastructure. 
To address this gap, \citet{fowl2022robbing} and \citet{curious} explored the \textit{malicious server} scenario, where the server actively manipulates model parameters to force clients data leakage.
Recent advancements in this domain have focused on overcoming batch size limitations \cite{diana2025cutting}, improving noise robustness \cite{scalemia} and reducing detectability \cite{garov2024hiding}.

Despite substantially advancing the practical power of gradient inversion, existing attacks cannot certify that a reconstruction matches the true underlying sample, unless they have access to some prior knowledge about the dataset.
In vision and language domains, this issue is often mitigated through human inspection; for example, \citet{mia} rely on crowd-sourced human evaluation to assess the recognizability of reconstructed images. In contrast, tabular data, which is central to many real-world federated learning applications~\citep{ProjectAIKYA,OgierduTerrail2025FedECA}, does not offer comparable semantic validation methods~\citep[Fig.~1]{tableak}. As noted by~\citet[Section~3]{tableak}, for an uninformed adversary it is often close to impossible to assess reconstruction quality in tabular settings, substantially complicating reliable privacy risk estimation. Figure~\ref{fig:verification} illustrates the difficulty to recognize erroneously reconstructed tabular data.

\begin{figure}
    \centering
    \includegraphics[width=1\linewidth]{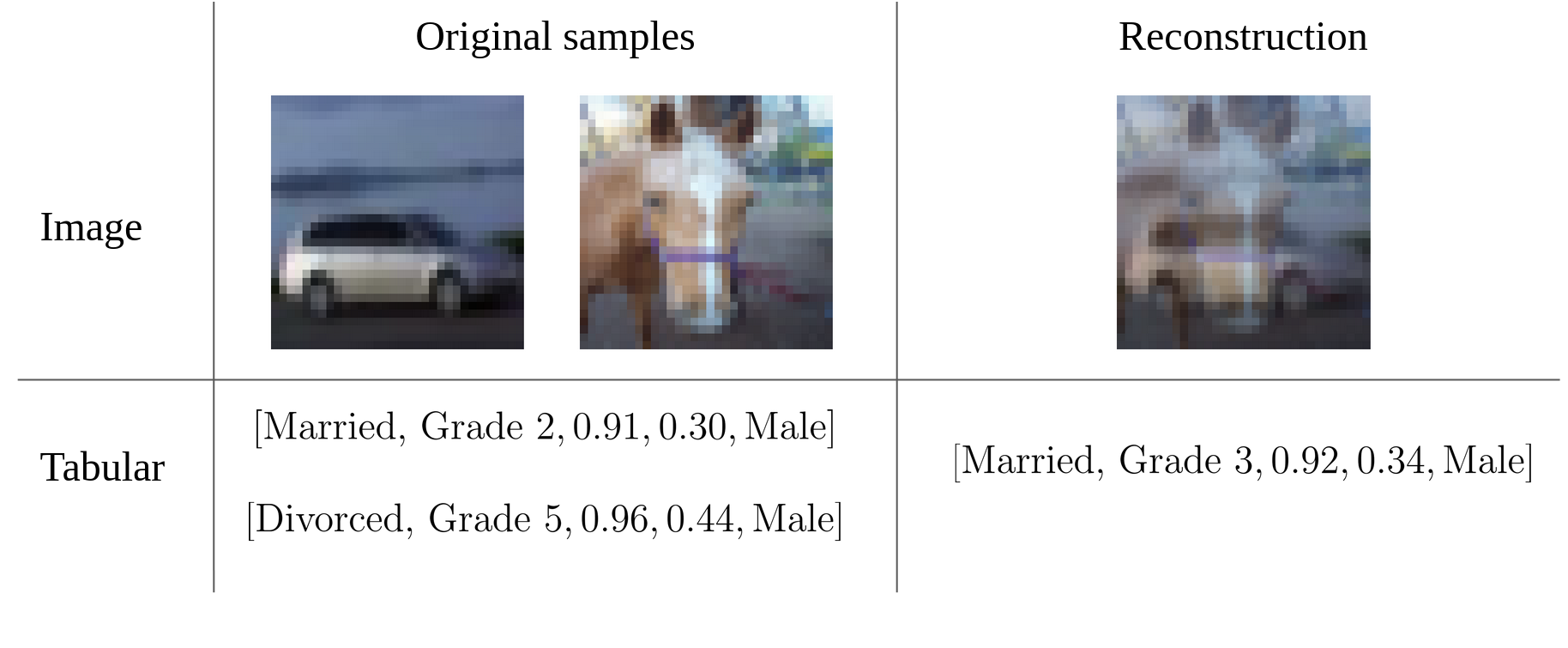}
    \caption{True samples and their erroneous reconstruction by CTP~\cite{diana2025cutting}.
    Top: images  from the CIFAR10~\cite{cifar10}. Bottom: tabular records from  ACS Income~\cite{income}.}
    \label{fig:verification}
\end{figure}


To date, only \citet{tableak} explicitly addresses this issue by proposing an entropy-based measure to quantify reconstruction certainty. However, their approach requires repeating the same optimization-based attack from ten to thirty times to obtain an ensemble of reconstructions, incurring significant computational overhead while still inheriting the batch-size limitations of optimization-based attacks in the honest-but-curious setting. By contrast, \citet{dimitrov2024spear} propose an analytical attack whose reconstructions can be verified with high probability, but their method requires the batch size to be smaller than the input dimension—a condition that is typically satisfied in vision and language settings but is impractical for tabular data.


To rigorously audit the privacy risk in tabular FL, we operate within the malicious server threat model. We present a novel reconstruction attack wherein the server, by crafting specific model parameters, successfully isolates individual samples within an aggregated batch, introducing a subspace verification method that allows the attacker to verify whether the isolation process is successful, eliminating the reliance on heuristics. Once the input features of a sample are recovered, our framework enables the exact reconstruction of the corresponding target value, resulting in a complete breach of the record. 
Our contributions are summarized as follows:
\begin{itemize}
    
    \item We introduce VGIA,  a novel analytical gradient inversion attack. Unlike prior heuristic methods, our approach certifies whether a reconstruction is correct, establishing a more rigorous baseline for privacy auditing.
    \item VGIA achieves exact reconstruction of both input features and---under a broad class of loss functions---target values. To the best of our knowledge, this is the first attack to provably recover continuous targets in regression settings.
    \item We empirically validate VGIA on both tabular and image datasets,  demonstrating exact record recovery even under large-batch aggregation and cross-domain applicability.
\end{itemize}

\section{Background and Related Work}
Federated Learning (FL) enables a set of clients \clients, coordinated by a central server, to collaboratively optimize a shared global model $\theta$. The training objective is to minimize the weighted average of the clients' empirical risks: 
\begin{equation}\label{eq:globalobjective} 
\min_{\model \in \real^d} \globallossfunction(\model) = \sum_{c\in \clients} p^c\loss(\model, \localdataset), 
\end{equation} 
where $\localdataset = \{(\x_i^c, y_i^c)\}_{i=1}^{n^c}$ denotes the local dataset of client $c$ 
and $p^c$ is a positive weight satisfying $\sum_{c\in \clients} p^c = 1$, typically proportional to the dataset size 
or set uniformly. 
The optimization proceeds in communication rounds. In each round, participating clients download the current global model, compute local updates using their data, and send these updates to the server. Depending on the algorithm, clients may perform either a single stochastic gradient update (FedSGD; \citealp{45187}) or multiple stochastic gradient steps (FedAvg; \citealp{mcmahan2017communication}).

Recent work shows that during FedSGD and FedAvg training, a server can reconstruct private client data via gradient inversion attacks (GIAs) by eavesdropping on exchanged messages in the \textit{honest-but-curious} setting or by crafting global model parameters in the \textit{malicious} setting.
Following~\citet{carletti_sok}, GIAs can be grouped into three families: optimization-based, generative model-based, and analytical-based. Below, we briefly review these approaches and highlight their limitations in terms of verifiability.

\paragraph{Optimization-Based attacks}
Optimization-based attacks were the first approach to recover client inputs from FL updates~\citep{dlg}. Starting from dummy inputs and labels, the attacker reconstructs data by minimizing the discrepancy between gradients computed on the dummy data and the observed pseudo-gradients. To improve reconstruction quality, several works decouple label and input recovery~\citep{idlg, e2egi, dimitrov2022data, geng2022generaldeepleakagefederated, ma2023instancewise}, often augmenting the objective with regularization terms or image priors~\citep{geiping_gi, yin_gi, dimitrov2022data, geng2022generaldeepleakagefederated}.
\citet{kariyappa23a_cocktail} propose separating per-sample gradient contributions via Independent Component Analysis, while~\citet{temporal} exploit information aggregated across multiple communication rounds. Nevertheless, reconstruction quality still degrades sharply for ImageNet batch sizes exceeding 256.
Under a malicious threat model,~\citet{wen2022fishing, gradfilt, shan2025geminio} propose attacks that selectively amplify per-sample gradient contributions, mitigating the obfuscating effects of large batch sizes. These methods assume the adversary has access to an auxiliary dataset drawn from a similar distribution.


\paragraph{Generative Model-Based attacks}
Optimization-based attacks become computationally prohibitive for high-dimensional data, e.g., high-resolution images. Consequently, recent work leverages generative models, including generative adversarial networks and diffusion models. 
In the honest-but-curious setting, instead of optimizing directly over the batch of inputs, the attacker trains a generative model to match the gradient of generated samples to the observed pseudo-gradient~\citep{GRNN, lti, ci-net, cgir, girg, FGLA, FGLA_jorunal, meng2025enhancedprivacyleakagenoiseperturbed, carletti2025guideenhancinggradientinversion}.
In the malicious setting, the server locally trains an encoder and a secret generative decoder using an auxiliary dataset. The encoder, instantiated as a crafted global model, either conceals the amplification of a target sample’s gradient~\citep{garov2024hiding} or facilitates reconstruction from a latent space~\citep{scalemia}. The decoder is then applied 
on the server to recover client inputs.

\paragraph{Analytical attacks}
Analytic attacks can recover \textit{exact} inputs by exploiting analytical properties of gradients on fully connected ReLU layers; we review the underlying theory  
in Sec.~\ref{subsec:preliminaries}. 
In the honest-but-curious setting,~\citet{zhu2021rgap} demonstrate the feasibility of such attacks on convolutional neural networks for batch size one.
Subsequent work~\citep{dimitrov2024spear, bakarsky2025spear++} extends these attacks to batch sizes up to 200 by exploiting the low-rank structure and sparsity of gradients, requiring a batch size smaller than the input dimension.
In the malicious setting,~\citet{curious, ZHANG2023119421} craft neurons that activate for a single sample per batch, but their approaches suffer fundamental accuracy limits, especially on low-dimensional data~\citep{diana2025cutting}. Other work~\citep{fowl2022robbing, fowl2023decepticons, chu2023panning, loki} uses pairs of neurons activating on the same set of samples except one, requiring prior knowledge of data distributions. \citet{diana2025cutting} remove this assumption via a binary search scheme, recovering up to 1024 ImageNet samples per batch in 10 rounds. In parallel,~\citet{mkor} decouple gradients by label, but requiring one sample per class in the batch.

\paragraph{Limitation on verifiability}
To assess the success of GIAs, prior work typically evaluates reconstructions \emph{against ground-truth samples} (available to the experimenter). For noisy reconstructions produced by optimization-based or generative attacks, image studies commonly report SSIM, PSNR, and LPIPS, while text studies often use exact match / token-level accuracy and sequence-overlap metrics such as BLEU or ROUGE. For analytic attacks, where the output is either an exact sample or a linear combination of samples, a common metric is the fraction of reconstructions that exactly match a true record (up to numerical tolerance).
As noted in~\cite{tableak}, these ground-truth-based metrics do not reflect the attacker’s perspective: without access to the original batch, the adversary lacks a per-reconstruction certificate of success.
Even an attack that consistently recovers  a small fraction of samples (e.g., $10\%$) can pose a serious privacy threat, but its practical impact is greatly reduced if the attacker cannot identify which reconstructions are correct. Ideally, GIAs should therefore provide an intrinsic verification mechanism. For vision and language data, verification may rely on human inspection, but this is largely infeasible for tabular records (Fig.~\ref{fig:verification}). Despite its practical importance, verifiability has been largely overlooked in the GIA literature; only a few works offer any form of verification, typically under restrictive conditions.

\citet{tableak} propose an optimization-based attack tailored to tabular data and introduce an entropy-based verification score that requires repeatedly solving the same optimization problem to obtain an ensemble of local minima. Feature-wise entropy is then computed across these minima, and the $25\%$ lowest-entropy features are deemed valid. However, this procedure is computationally expensive and still yields noisy, incomplete reconstructions with missing features, while providing only a heuristic notion of verification.
Although analytical attacks can yield \textit{exact} reconstructions rather than noisy ones, verification remains non-trivial. The approaches of~\citet{dimitrov2024spear, bakarsky2025spear++} require the batch size to be smaller than the input dimension, a condition rarely met in tabular datasets. 
Meanwhile, the attack of \citet{diana2025cutting} is verifiable only if the attacker has prior knowledge of the minimum separation between projected inputs within a batch---an assumption that is difficult to justify in practice.





\section{Threat model}\label{sec:threat_model}

We consider a threat model in which the central server acts as a \textit{malicious} adversary and actively modifies the model parameters sent to clients. This setting is consistent with prior work on malicious-server attacks in Federated Learning \citep{wen2022fishing, fowl2022robbing, curious, scalemia, diana2025cutting}.

Clients are assumed to trust the server and therefore apply the received model parameters without verification before computing and sending their updates. At the same time, the server cannot artificially modify the architecture of the model, for example by adding layers, modifying layers' connections  \cite{fowl2022robbing, loki} or activation functions. In contrast to \citet{fowl2022robbing, wen2022fishing}, we assume that the attacker has no auxiliary knowledge of the data distribution beyond known bounds on the input features. Furthermore, unlike \citet{curious, wen2022fishing, garov2024hiding, scalemia}, we do not assume access to any auxiliary or public datasets. In the following, we assume that clients update the global model using FedSGD with full-batch updates.

\section{Our attack}

\subsection{Preliminaries}\label{subsec:preliminaries}
\paragraph{Analytic Gradient Inversion} Our work builds upon the fundamental observation that the gradients of a fully connected (FC) layer followed by a ReLU activation can contain sufficient information to analytically recover input data \citep{phong, geiping_gi, fowl2022robbing, diana2025cutting}. 


Consider a model whose first (attacked) layer is a linear layer with $N$ neurons, parameterized by weights $\W^a \in \mathbb{R}^{N \times d}$ and biases $\b^a \in \mathbb{R}^N$. For neuron $i$, the post-activation output is
\[
z_i^a = \mathrm{ReLU}\!\left(\W_i^a \x + b_i^a\right),
\]
where $\x \in \mathbb{R}^d$ denotes the layer input. For a training example $(\x_j, y_j)$, let $z_{i,j}^ a$ denote the activation of neuron $i$. The gradients of the per-sample loss $\loss_j$ with respect to the weights and bias of neuron $i$ are
\begin{align}
 \frac{\partial \loss_j}{\partial \W_i^a}
& = \frac{\partial \loss_j}{\partial z_{i}^ a}\frac{\partial z_{i}^a}{\partial \W_i^a}
= \frac{\partial \loss_j}{\partial z_{i}^a}\, \x_j\, \mathbbm{1}_{z_{i,j}^a>0}, \label{eq:wi_zi}\\
\frac{\partial \loss_j}{\partial b_i^a}
& = \frac{\partial \loss_j}{\partial z_{i}^a}\frac{\partial z_{i}^a}{\partial b_i^a}
= \frac{\partial \loss_j}{\partial z_{i}^a}\, \mathbbm{1}_{z_{i,j}^a>0}. \label{eq:bi_zi}
\end{align}
The indicator $\mathbbm{1}_{z_{i,j}>0}$ encodes whether $\x_j$ activates neuron~$i$. Therefore, whenever $z_{i,j}>0$ and $\nabla_{b_i^a}\loss_j \neq 0$, the input can be recovered exactly from the gradients via
\begin{equation}
\label{eq:single_point_reconstruction}
\x_j = \frac{\partial \loss_j}{\partial \W_i^a}\left( \frac{\partial \loss_j}{\partial b_i^a} \right)^{-1}.
\end{equation}

In FedSGD, the server does not observe per-sample gradients; instead, it receives the gradient of the batch-averaged loss $\loss=\frac{1}{B}\sum_{j=1}^B \loss_j$, computed over the batch $\mathcal{B}=\{(\x_j,y_j)\}_{j=1}^B$. As shown in prior work~\citep{ZHANG2023119421, diana2025cutting}, taking the ratio of the corresponding gradient components as in~\eqref{eq:single_point_reconstruction} allows the server to recover a linear combination of the batch inputs (we call this an \emph{observation}):
\begin{align}
\o_i & \;=\; \frac{\partial \loss}{\partial \W_i^a}\left(\frac{\partial \loss}{\partial b_i^a}\right)^{-1}
\;=\; \sum_{j=1}^B \alpha_j \x_j, \label{eq:observation}\\
\alpha_j & \;=\; \frac{\frac{\partial \loss_j}{\partial b_i^a}}{\sum_{k=1}^B \frac{\partial \loss_k}{\partial b_i^a}}.
\end{align}
Only the samples $\x_j$ that activate neuron $i$ contribute to observation~$\o_i$, i.e., they are the only ones with $\alpha_j \neq 0$. We refer to this subset as the set of active samples and denote the set of their indices by $\activeset_i = \{j \in [B] \mid \W_i^a \x_j + b_i^a > 0\}$.

Disentangling individual samples from the aggregate $\o_i$ is the central challenge in analytic gradient inversion.

\paragraph{Geometric Isolation}
To address this, \citet{diana2025cutting} adopt a geometric view of ReLU-based leakage through the activation boundary. For neuron $i$, the weights $\W_i^a$ and bias $b_i^a$ define the hyperplane
$\mathcal{H}_i=\{\x\in\mathbb{R}^d \mid \W_i^a \x + b_i^a = 0\}$,
which partitions the batch into inputs on the ReLU-active side $\W_i^a \x + b_i^a > 0$ and on the inactive side $\W_i^a \x + b_i^a \le 0$. Fixing $\W_i^a$ and varying $b_i^a$ translates $\mathcal{H}_i$ along its normal direction, yielding a family of parallel hyperplanes that \emph{sweep} across the batch. Building on this interpretation, \citet{diana2025cutting} propose the CTP algorithm, which searches over $b_i^a$ to identify thresholds $\hat b_{i,1}^* < \dots < \hat b_{i,B+1}^*$ such that the slab between two consecutive hyperplanes contains exactly one batch element $\x_k$. 
The resulting sequence of gradient-ratio observations forms a linear system; under Assumption~4.1 in~\citet{diana2025cutting} which states that $\partial \loss_j/\partial b_i^a$ does not depend on $b_i^a$, this system can be solved to reconstruct the entire batch.\footnote{
While the presentation in~\citep{diana2025cutting} suggests that all inputs must be isolated (up to $\varepsilon$) before the resulting triangular system can be solved sequentially, their implementation already performs reconstructions even when some slices are wider than $\varepsilon$.
} Moreover, multiple hyperplanes can be queried within the same FedSGD round by attacking different neurons.

It is important to note that while CTP can detect when a slab contains no samples (and thus avoid placing further hyperplanes there), it cannot distinguish whether a non-empty slice contains one or multiple samples. Consequently, it keeps bisecting any slice that contains at least one point until two consecutive hyperplanes are within distance $\epsilon$ (a user-chosen parameter). Figure~\ref{fig:ctp_ok} illustrates a case where CTP correctly recovers three inputs after 6 rounds (3 hyperplanes per round, 18 hyperplanes total). In contrast, Fig.~\ref{fig:ctp_miss} shows a failure  where two inputs lie closer than $\varepsilon$; CTP stops before separating them, leading to incorrect reconstruction.

From the discussion above, two  key limitations of CTP emerge:
\begin{itemize}
    \item \textbf{Unverifiability}: CTP cannot certify the correctness of the reconstructed samples and produce systematic errors when the projections of two inputs onto the direction $\W_i^a$ are closer than $\varepsilon$. In a realistic blind threat model, the adversary lacks the prior knowledge needed to tune $\varepsilon$ empirically.
    \item \textbf{Inadaptability}: CTP bisects every non-empty slab down to resolution $\varepsilon$, potentially spending multiple attack rounds to further refine samples that were already isolated early on.
\end{itemize}

\begin{figure}[tb]
    \centering
    \subfloat[\label{fig:ctp_ok}]{%
        \includegraphics[width=0.45\linewidth]{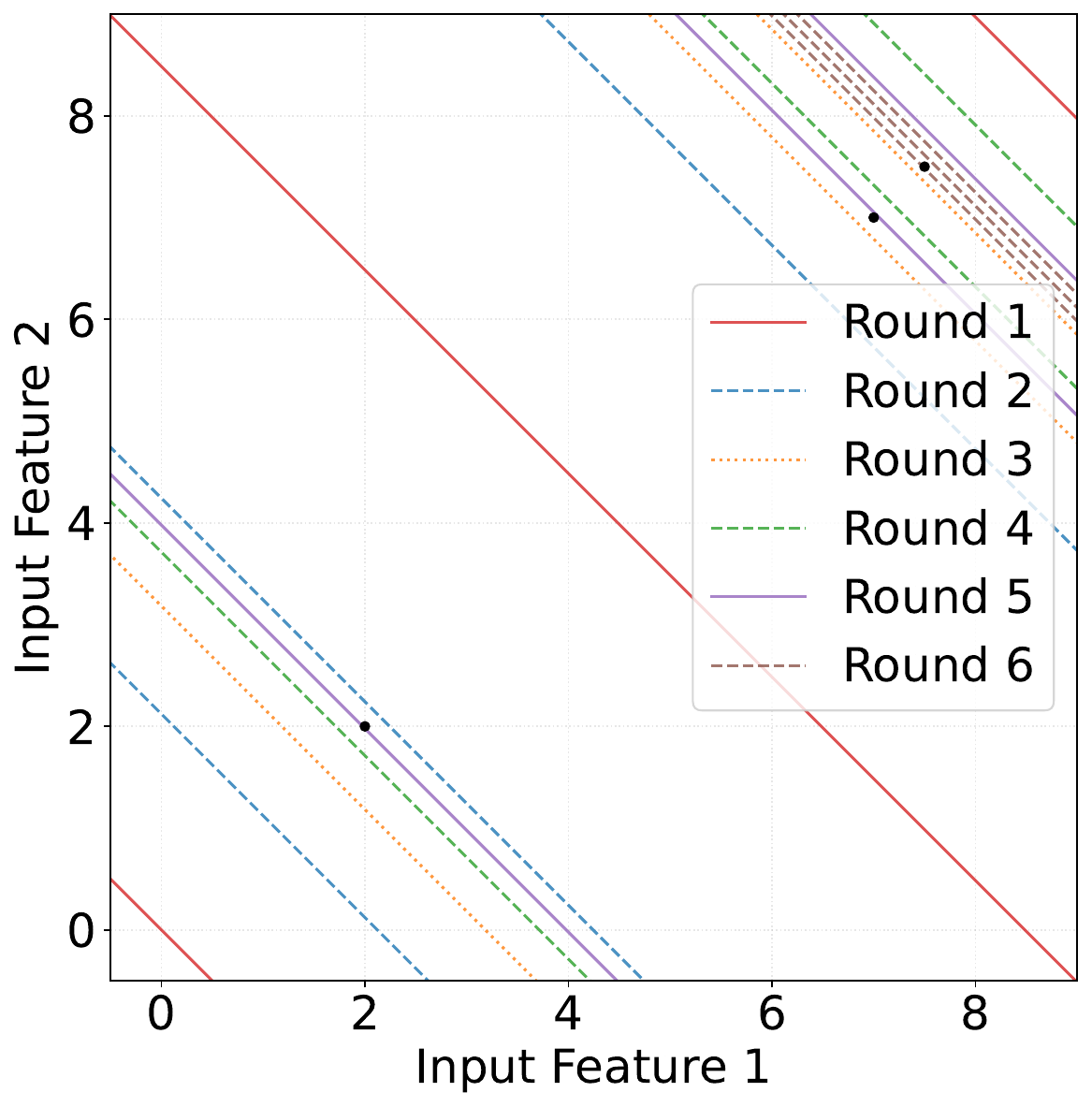}
    }\hfill 
    \subfloat[\label{fig:new_ok}]{%
        \includegraphics[width=0.45\linewidth]{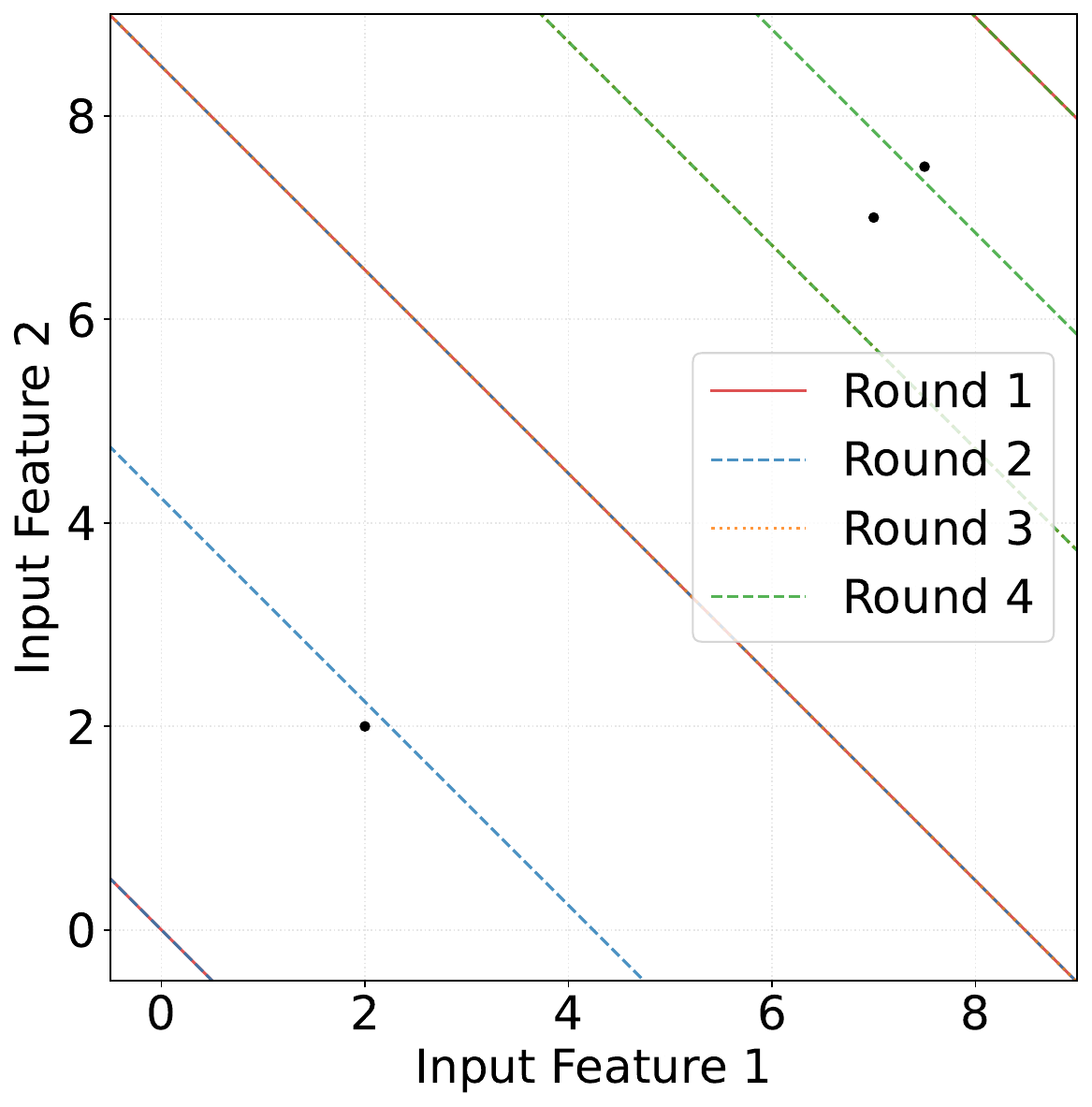}
    }

    \subfloat[\label{fig:ctp_miss}]{%
        \includegraphics[width=0.45\linewidth]{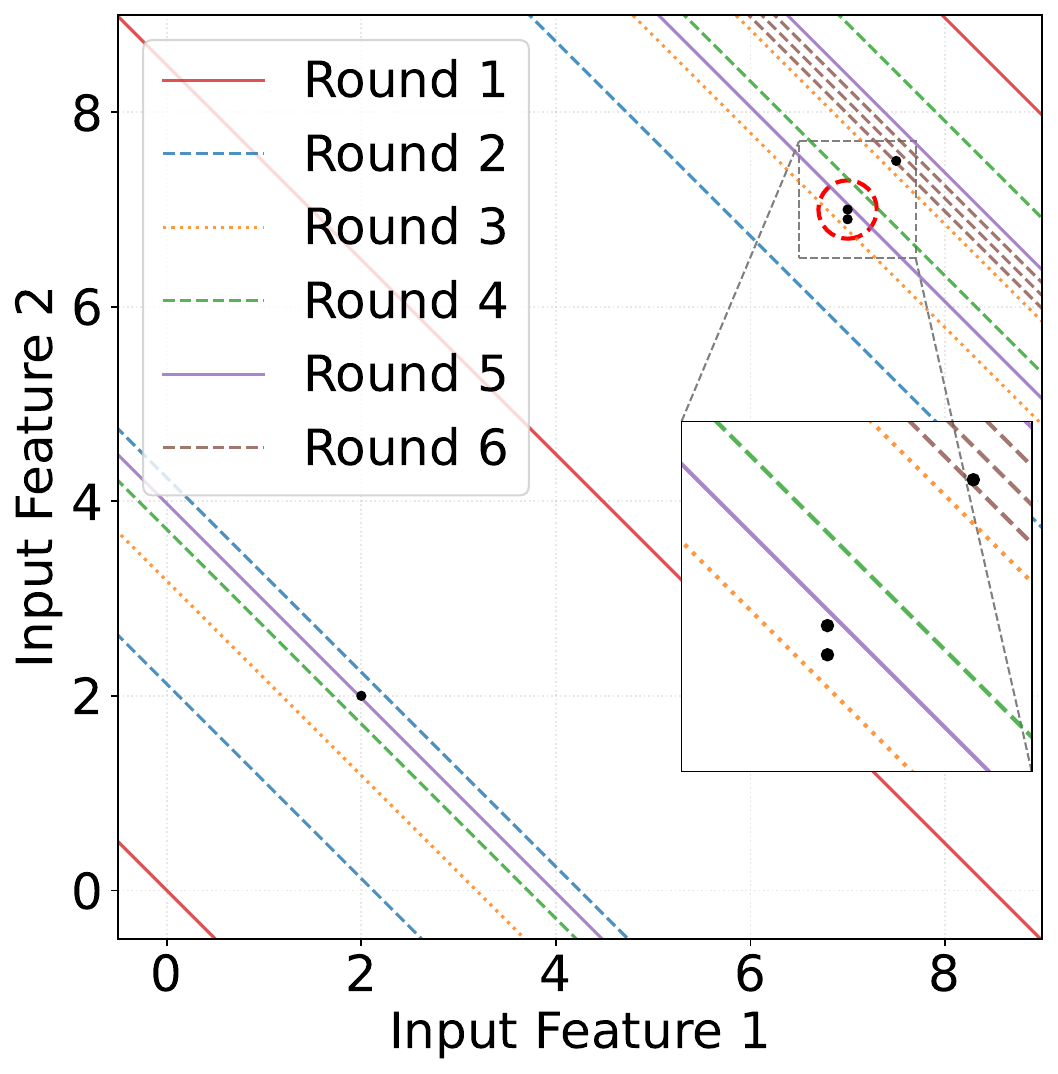}
    }\hfill
    \subfloat[\label{fig:new_adv}]{%
        \includegraphics[width=0.45\linewidth]{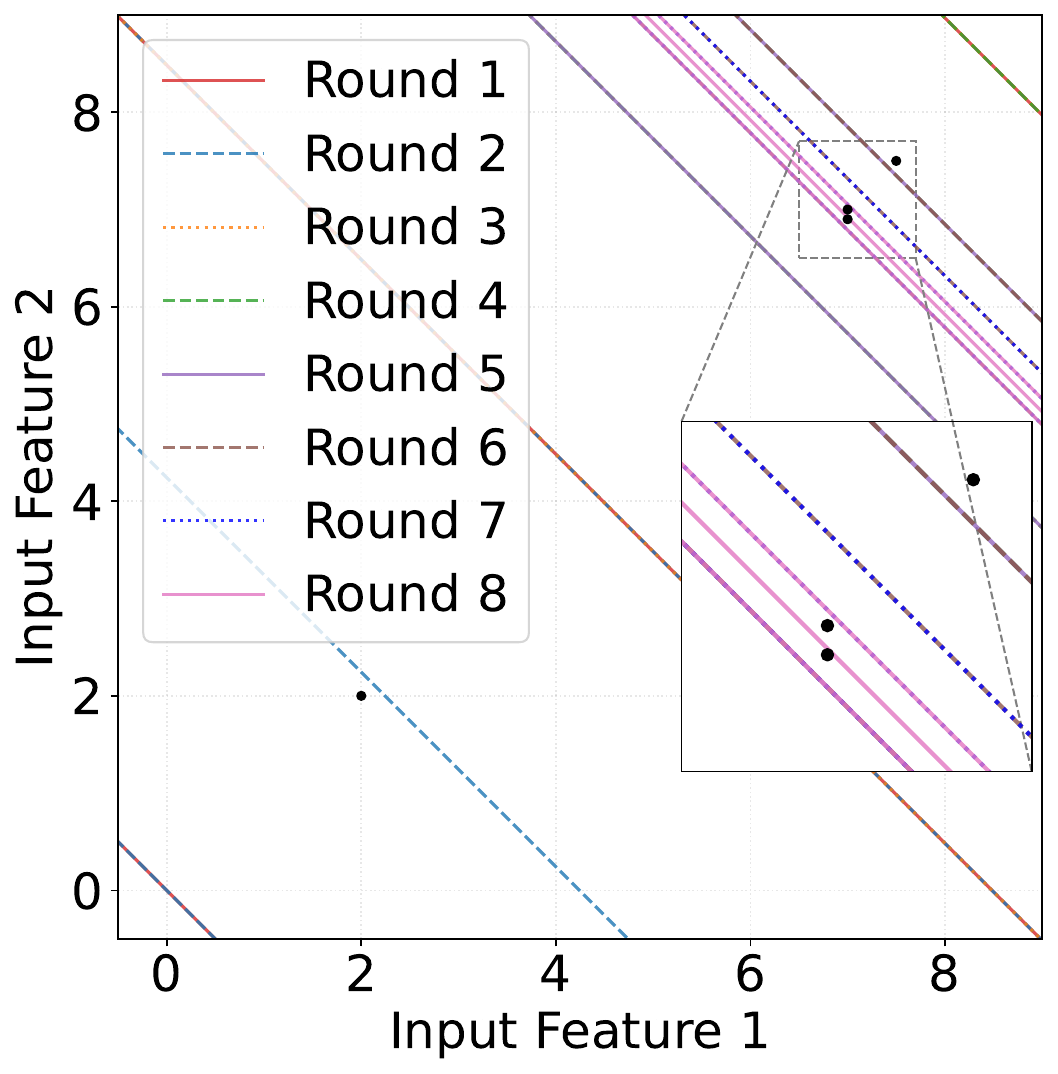}
    }

    \caption{A two-dimensional example of the search process. Figure~\ref{fig:ctp_ok} illustrates the process proposed by \citet{diana2025cutting}, while Figure~\ref{fig:new_ok} depicts our proposed method. In this instance, our method requires fewer rounds than their approach, as we avoid searching up to $\varepsilon$ distance across the entire space. Conversely, when $\varepsilon$ not well-tuned, Figure~\ref{fig:ctp_miss} shows that the CTP attack fails, whereas our attack (Figure~\ref{fig:new_adv}) continues searching until all data points are exactly separated.}
    \label{fig:main_fig}
\end{figure}
\subsection{Attack method}\label{sec:attack_method}

Our new dynamic attack addresses the two limitations of CTP. First, it introduces an algebraic step that makes isolation events verifiable: the attacker can certify when a slice contains a single input. Second, once a slice is certified to contain a single input, our algorithm removes it from further refinement and may allocate future hyperplane queries adaptively across unexplored regions, avoiding unnecessary dissections and reducing the number of communication rounds required for full-batch reconstruction. 

Figure~\ref{fig:new_ok} illustrates that, by correctly identifying slices that contain a single input, our attack reconstructs all samples in only $4$ rounds ($12$ hyperplanes), compared to $6$ rounds ($18$ hyperplanes) for CTP (Fig.~\ref{fig:ctp_ok}). Moreover, Fig.~\ref{fig:new_adv} shows that our protocol adapts the number of rounds and hyperplanes to the specific batch: in this example it continues for $8$ rounds to isolate all inputs, whereas CTP (Fig.~\ref{fig:ctp_miss}) stops too early and produces an incorrect reconstruction.

We present the attack for an $L$-layers fully connected neural network with one scalar output and ReLU activation functions. The set of parameters is denoted by $\W$. The extension to multi-dimensional regression and multi-class classification tasks is described in Appendix~\ref{app:class_ext}.

As above, the first layer (the target of our attack) is parameterized by weights $\W^a \in \real^{N\times d}$ and biases $\b^a \in \real^ N$, followed by a ReLU activation. 
The last layer is parameterized by weights $\Wout \in \real^N$ and a bias $b^{(L)} \in \real$. 

\paragraph{Verifiable isolation}
The first step of our attack requires the attacker to spatially isolate each sample, building upon the geometric framework established by \citet{diana2025cutting}. 
Within a single round, the attacker can test multiple parallel hyperplanes with normal $\mathbf w$ by targeting different neurons and setting their weight vectors to $\mathbf w$ (e.g., $\W_i^a=\W_{i'}^a=\mathbf w$ when attacking neurons $i$ and $i'$), while assigning them different biases (so $b_i^a \neq b_{i'}^a$). For simplicity, we assume the attacker targets all neurons in the layer.
The attacker seeks to identify a sequence of bias values $\hat{b}_{1}^*<\dots <\hat{b}_{B+1}^*$ for the attack layer such that  each consecutive slice contains a single sample.

We first analyze the gradient of the loss with respect to neuron $i$'s weights  $\gradlosswiatt$. 
\begin{align}
    \gradlosswiatt &= \frac{1}{B} \sum_{j=1}^B \frac{\partial \loss_j}{\partial z_{i}^a}\, \x_j\, \mathbbm{1}_{z_{i,j}^a>0} = \frac{1}{B} \sum_{j \in \activeset_i} \frac{\partial \loss_j}{\partial z_{i}^a}\, \x_j\\
    & = \frac{1}{B} \sum_{j \in \activeset_i} \frac{\partial \loss_j}{\partial b_{i}^a}\, \x_j =\frac{1}{B} \sum_{j \in \activeset_i}
    \frac{\partial \loss_j}{\partial z^L}\frac{\partial z^L}{\partial b_i^a}\x_j, \label{eq:weight_gradient}
\end{align}
where the first equality follows from~\eqref{eq:wi_zi} and the third from~\eqref{eq:bi_zi}.

In general, the  gradient $\frac{\partial z^L}{\partial b_i^a}$  from the network output to neuron $i$ in layer $a$ depends on the input $\x$, because downstream ReLU gates may switch on/off as $\x$ changes. Concretely, $\frac{\partial z^L}{\partial z_i^a}(\x)$ is typically $\x$-dependent. In our threat model, however, the attacker can choose the parameters of the subsequent layers $(\W^{a+1:L}, \b^{a+1:L})$ so as to keep the post-$a$ network in a fixed linear region for all $\x$ of interest. Since ReLU networks are piecewise linear, once the activation pattern of all ReLUs after layer $a$ is fixed, $z^L$ becomes an affine function of $z^a$, and therefore the Jacobian entry $\frac{\partial z^L}{\partial z_i^a}$ is a constant. The attacker can then enforce the following condition:
\begin{equation}
\label{eq:attack}
\frac{\partial z^L}{\partial z_i^a}(\x)=g_i \neq 0,\qquad \forall \x.
\end{equation}
Intuitively, the attacker selects $\W^{a+1:L}$ to predefine which downstream ReLUs always pass the signal and which always block it, independently of $\x$. A simple choice is to use positive weights and biases so that all downstream pre-activations remain positive, making all ReLUs always active; more generally, the attacker can realize a prescribed, input-independent gating pattern using mixed-sign weights together with biases that keep each downstream pre-activation strictly on one side of zero, yielding realistic-looking weight distributions.

As $\frac{\partial z^L}{\partial z^a_i}= \frac{\partial z^L}{\partial b^a_i}$ for all inputs that activate neuron $i$, we have that \eqref{eq:attack} guarantees
\begin{equation}\label{eq:lin_comb_err}
    \frac{\gradlosswiatt}{g_i}= \frac{1}{B} \sum_{j \in \activeset_i} \frac{\partial \loss_j}{\partial z^{L}} \x_j,
\end{equation}
where $\frac{\partial \loss_j}{\partial z^{L}}$ depends only on the loss function and on the sample $\x_j$.

Consider two consecutive bias values $\hat b_k < \hat b_{k+1}$ tested at two attacked neurons. Without loss of generality, assume these neurons are $i$ and $i+1$, and we call the region delimited by the two hyperplanes $(\mathbf w,\hat b_k)$ and $(\mathbf w,\hat b_{k+1})$ the \emph{$k$-th slice}. From~\eqref{eq:lin_comb_err} and the analogous expression for neuron $i+1$, the server can compute what we call the \emph{slice vector} $\slice_k$, which is a linear combination of the inputs lying between these two hyperplanes:
\begin{align}
   \slice_k &=  \frac{\frac{\partial \loss}{\partial \W_{i+1}^a}}{g_{i+1}} - \frac{\frac{\partial \loss}{\partial \W_{i}^a}}{g_i} \nonumber\\
   &= \frac{1}{B} \sum_{\x_j \in \activeset_{i+1}} \frac{\partial \loss_j}{\partial z^{L}} \x_j - \frac{1}{B} \sum_{\x_j \in \activeset_{i}} \frac{\partial \loss_j}{\partial z^{L}} \x_j  \nonumber \\
    &= \frac{1}{B} \sum_{\substack{j \in \activeset_{i+1} \\ j \notin \activeset_{i}}} \frac{\partial \loss_j}{\partial z^{L}} \x_j \label{eq:hp_diff}.
\end{align}


Note how the rescaling in \eqref{eq:lin_comb_err} plays a central role in the attack, as it aligns gradient observations across consecutive hyperplanes in \eqref{eq:hp_diff} and enables the exact cancellation of sample contributions that activate both neurons. Without this step, differences between consecutive gradients would retain contributions from all samples in $\activeset_i$, since the scaling factors $g_i$ and $g_{i+1}$ generally differ. 


When the interval between two consecutive hyperplanes contains no data points, the corresponding active sets $\activeset_i$ and $\activeset_{i+1}$ coincide, producing a null slice vector $\slice_k=\mathbf{0}\in\real^d$. Such empty slices indicate that no further exploration is required in that region. Accordingly, in subsequent communication rounds, the server updates the bias values  to explore only those intervals where non-empty slices are detected, effectively pruning the search space.

Moreover, as we show next, if the attacker allocates $q\ge 3$ hyperplanes to probe a non-empty slice in the next attack round, then after that round it can certify whether the slice contains at most $q-2$ samples. Whenever this condition holds, the attacker can exactly decode every sample in the slice, without any further exploration rounds.

Indeed, consider the non-empty slice at attack round $t-1$ identified by the bias values $\hat b_k^{t}$ and $\hat b_{k+1}^{t}$ and the corresponding slice vector $\slice_k^t$. At the next attack round, the attacker will test $q$ bias values in the interval $[\hat b_k^{t}, \hat b_{k+1}^t]$, that is 
$\hat b_{k}^{t+1} (= \hat b_{k}^{t}) < \hat b_{k+1}^{t+1} < \dots \hat b_{k+q-1}^{t+1} (=\hat b_{k+1}^{t})$, and compute the corresponding slice vectors $\slice_k^{t+1}, \slice_{k+1}^{t+1}, \dots \slice_{k+q-1}^{t+1}$.

\begin{proposition}
\label{prop:sample_isolation}
Consider the slice delimited by the hyperplanes $\mathbf w \x + \hat b_k^t=0$ and $\mathbf w \x + \hat b_{k+1}^t=0$. Assume all inputs in the slice  are linearly independent,
 all level sets of $\frac{\partial \loss_j }{\partial z^L}$ are Lebesgue-null, and $\W^L$ is drawn at each attack round from an absolutely continuous distribution. The slice contains at most $q-2$ samples with probability $1$ if and only if 
\begin{equation}
\label{eq:subspace_condition}
\slice_k^t \in \mathrm{span}\left(\slice_k^{t+1}, \slice_{k+1}^{t+1}, \dots \slice_{k+q-1}^{t+1}\right).
\end{equation}
Moreover, when this condition is satisfied, the number of samples equals the number of non-null slice vectors~$\slice_h^{t+1}$.
\end{proposition}

Proposition~\ref{prop:sample_isolation} tells the attacker whether the slice must be further explored with additional hyperplanes, or---when~\eqref{eq:subspace_condition} holds---whether all samples in the slice have been correctly isolated and can therefore be reconstructed using the procedure described below.

    
\paragraph{Reconstruct the input}
Once a sample has been successfully isolated, the attacker can recover its input features analytically. Without loss of generality, suppose that at communication round $t$ the server isolates an input $\x_k$ within the slice delimited by two consecutive bias values $\hat b_k^*$ and $\hat b_{k+1}^*$, tested at neurons $i$ and $i+1$.
From~\eqref{eq:hp_diff}, the corresponding slice vector is
\begin{equation}\label{eq:isolation}
    \slice_k = \frac{1}{B} \frac{\partial \loss_k}{\partial z^{L}} \x_k.
\end{equation}
Recovering $\x_k$ therefore reduces to estimating the corresponding scaling factor $\beta_k \coloneqq \frac{1}{B}\,\frac{\partial \loss_k}{\partial z^{L}}$. We can obtain $\beta_k$ by applying the same reasoning that leads to~\eqref{eq:hp_diff} to the gradients with respect to the bias parameters $\b^a$ of the attacked layer, which yields:
\begin{align}  \frac{\frac{\partial \loss}{\partial b_{i+1}^a}}{g_{i+1}} - \frac{\frac{\partial \loss}{\partial b_{i}^a}}{g_i} 
   &= \frac{1}{B} \sum_{j \in \activeset_{i+1}} \frac{\partial \loss_j}{\partial z^{L}} - \frac{1}{B} \sum_{j \in \activeset_{i}} \frac{\partial \loss_j}{\partial z^{L}} \nonumber \\
    &= \sum_{\substack{j \in \activeset_{i+1} \\ j \notin \activeset_{i}}}\frac{1}{B} \frac{\partial \loss_j}{\partial z^{L}}  \label{eq:rec_coeff}.
\end{align}
Hence for a slice containing a single sample, we can compute $\beta_k$ as $\frac{\frac{\partial \loss}{\partial b_{i+1}^a}}{g_{i+1}} - \frac{\frac{\partial \loss}{\partial b_{i}^a}}{g_i}$ and 
correctly reconstruct the private input vector as $\hat{\x}_k=\x_k=\slice_k^t / \beta_k$.

Finally, to reconstruct the target, we provide an optimization-based solution which is detailed in Appendix~\ref{app:reconstruction_target}.

\subsection{Final Algorithm}

To conclude, the full procedure is presented in Algorithm~\ref{algo:parallelattack}. 
As assumed in Section~\ref{sec:threat_model}, the attacker knows bounds on the input features, which induce bounds on the projections $\mathbf{w}\x$ and therefore define an initial search interval $[l,u]$ for the bias values.\\
 The procedure SetHyperplanes (Alg.~\ref{algo:set_hp_biases}) constructs the $N$ hyperplanes queried by the attacker, one per attacked neuron.
In the first communication round, these hyperplanes partitions $[u,v]$ into 
$N-1$ slices.
From the second round onward, 
each subinterval is explored by at least three hyperplanes: two positioned at the boundaries of the interval $[l_{\slice^{t-1}_i}, u_{\slice^{t-1}_i}]$, corresponding to slice $\slice^{t-1}_i$, and at least one additional hyperplane inside the interval to progressively reduce the search space (Algorithm~\ref{algo:set_hp_biases}). 
The boundary re-evaluation is necessary because, contrary to the assumptions made by \citet{diana2025cutting}, in our setting the model's parameters $\W^{a:L}$ are not frozen, but can be updated at every round (making our attack also less detectable). Consequently, the gradients for the same neuron $i$ in  \eqref{eq:weight_gradient} 
can vary across different rounds $t$. \\
The procedure \textsc{CheckIsolation} (Alg.~\ref{algo:check_isolation} in Appendix~\ref{app:algo}) uses Proposition~\ref{prop:sample_isolation} to certify when a slice contains exactly one sample; such slices are added to $\mathcal R$ and are ready for reconstruction.
It also detects empty slices via a null slice vector, in which case the corresponding interval requires no further exploration.



\begin{algorithm}[t]
\caption{VGIA}\label{algo:parallelattack}
\textbf{Input}: 
 $[l, u]$ initial search space of bias values, the set of attack rounds $\mathcal{T}=\{1,...,T\}$
\begin{algorithmic}[1]
    \STATE $\mathcal I \leftarrow ([l, u],)$ 
    \STATE Draw a vector $\mathbf{w}\in \real^d$ and set each row of $\W^a$ to $\mathbf{w}$.
    \FOR{$t \in \{1,...,T\} $ }
        \STATE $\mathcal{S} \leftarrow ()$;$\quad$ ($\b^a$, $\mathcal I$) $\leftarrow \mathrm{SetHyperplanes}(\mathcal{I})$\label{line:udpate_hyperplane}
        \STATE Server sends malicious parameters $\theta$ to the client
        \STATE Server receives gradient updates $\frac{\partial{\loss}}{\partial{\W^{a}}}$ and $\frac{\partial{\loss}}{\partial{\b^a}}$
        \STATE Server computes $\slice_i^{t}$ using \eqref{eq:hp_diff} and $\beta_i^t$ using \eqref{eq:rec_coeff}
        \STATE $\mathcal{S} \leftarrow (\mathcal{S},( (b_1^a,b_{2}^a, \slice_1^t, \beta_1^t),\dots, (b_{N-1}^a,b_{N}^a, \slice_{N-1}^t, \beta_{N-1}^t) ))$
        \STATE ($\mathcal{I}, \mathcal{R}) \leftarrow \mathrm{CheckIsolation}(\mathcal{I}, \mathcal{S})$
        \STATE Server uses $\mathcal R$ to reconstruct inputs and targets from \eqref{eq:hp_diff}, \eqref{eq:isolation}, \eqref{eq:rec_coeff}, and \eqref{eq:target_rec}
    \ENDFOR
\end{algorithmic}
\end{algorithm}

\begin{algorithm}[h]
\caption{SetHyperplanes}\label{algo:set_hp_biases}
\textbf{Input}: search intervals $\mathcal{I}=\{([l_i,u_i],\slice_i) \mid i=1, \dots, |\mathcal {I}|\}$

\begin{algorithmic}[1]
    \STATE sort $\mathcal{I}$ so that $l_1 < l_2 < \dots < l_{|\mathcal I|}$
    \STATE $M\leftarrow \min(|\mathcal{I}|, \lfloor N / 3 \rfloor);\quad r \leftarrow N \bmod M;\quad j \leftarrow 1$
    \FOR {$k=1,\dots,M$}
        \STATE $q \leftarrow \lfloor N/M \rfloor + \mathbbm{1}_{k\leq r}$
        
        \IF{$k > 1$ \AND $l_{\slice_k^t} = u_{\slice_{k-1}^t}$} 
            \STATE $\b^a_{j:j+q-1} \leftarrow \{ l_{k} + i \frac{u_{k} - l_{k}}{q} | i = 1, \dots, q \}$
            \STATE $j \leftarrow j+q-1$
        \ELSE
            \STATE $\b^a_{j:j+q-1} \leftarrow \{ l_{k} + i \frac{u_{k} - l_{k}}{q - 1} | i = 0, \dots, q-1 \}$
            \STATE $j \leftarrow j+q-1$
        \ENDIF
    \ENDFOR
    \STATE Return $\b^a$, $\mathcal I$
\end{algorithmic}
\end{algorithm}

\section{Experimental Evaluation}
\paragraph{Setup} We evaluate the attacks on four benchmarks:  three tabular datasets---ACS Income~\citep{income} and  King County Housing~\citep{datasetlink} for regression and Human Activity Recognition Using Smartphones (HARUS)  for classification--- and one image dataset CIFAR10 \citep{cifar10}. 
Dataset details are provided in Appendix~\ref{app:dataset}.
Each client trains a three-layer fully connected neural network (FC-NN) with $N=1000$ neurons in the first hidden layer and 100 neurons in the second layer.\footnote{FC-NNs are commonly used for tabular data and have also been considered in prior GIAs on tabular data~\citep{tableak}.}
The clients update their models using FedSGD with full-batch updates. \\
Due to space constraints, we report results on ACS Income in the main text and defer additional results, including those for FedAvg with multiple local updates, to Appendix~\ref{app:additional_results}.
\paragraph{Baselines} 
We compare our method (VGIA) with the CTP attack~\citep{diana2025cutting}, the state-of-the-art analytical attack under malicious threat model (Sec.~\ref{sec:threat_model}). 
As detailed in Sec.~\ref{subsec:preliminaries}, CTP has a user-chosen parameter $\varepsilon$ to stop the search.
Let $\mathcal D$ be the attacked client local dataset and $\varepsilon_{\mathbf w}$ be the minimum absolute difference between the projections of any two distinct samples in $\mathcal D$ along direction $\mathbf w$, i.e., $\varepsilon_{\mathbf w} = \min_{\mathbf x_1 \neq \mathbf x_2 \in \mathcal D} \left| \mathbf w^\top (\mathbf x_1 - \mathbf x_2) \right|$. We test CTP under three settings of $\varepsilon$: $\varepsilon <\varepsilon_{\mathbf w}$, $\varepsilon = \varepsilon_{\mathbf w}$, and $\varepsilon >\varepsilon_{\mathbf w}$, denoted by CTP$_<$, CTP$_=$, and CTP$_>$, respectively.\\
Each method is tested over three random seeds. 
For each seed, we sample (i) a random normal vector $\mathbf w$ defining the direction of the queried hyperplanes and (ii) a target client's dataset $\mathcal D$  drawn from the full dataset. To ensure a fair comparison, we keep $\mathbf w$ and $\mathcal D$  the same for both VGIA and CTP. The detailed attack configurations are presented in Appendix~\ref{app:exp_setup}.


\begin{figure*}[tb]
    \centering
    \subfloat[VGIA V.S. CTP$_{=}$\label{fig:income_good_epsilon}]{%
        \includegraphics[width=0.3\linewidth]{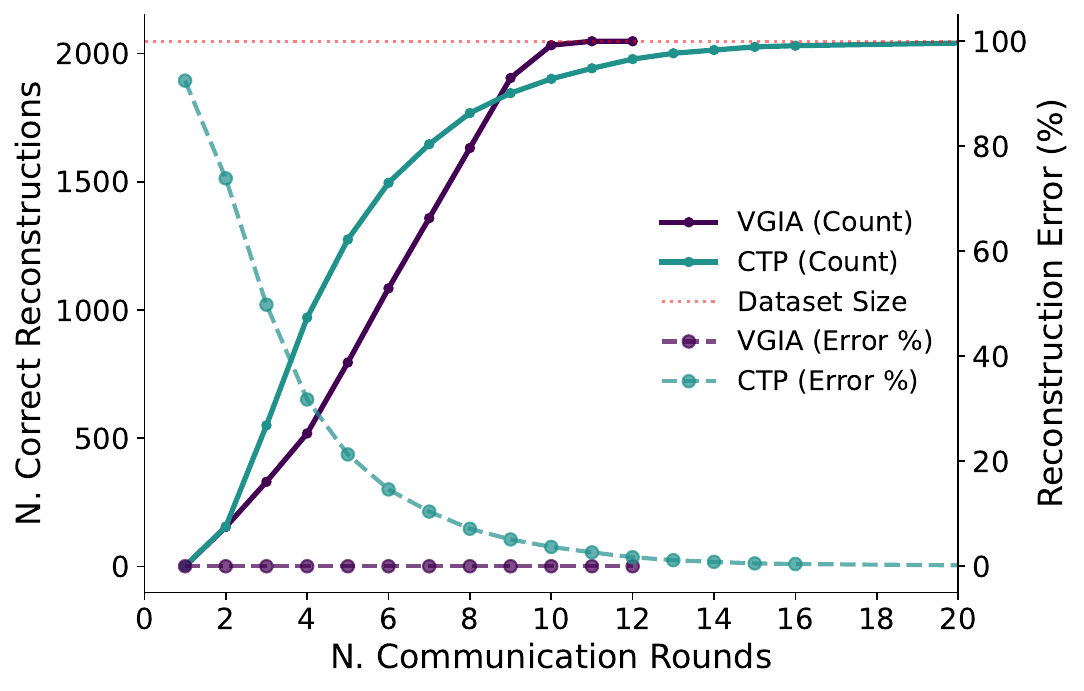}
    }\hspace{0.025\linewidth}
    \subfloat[VGIA V.S. CTP$_{>}$ \label{fig:income_bad_epsilon}]{%
        \includegraphics[width=0.3\linewidth]{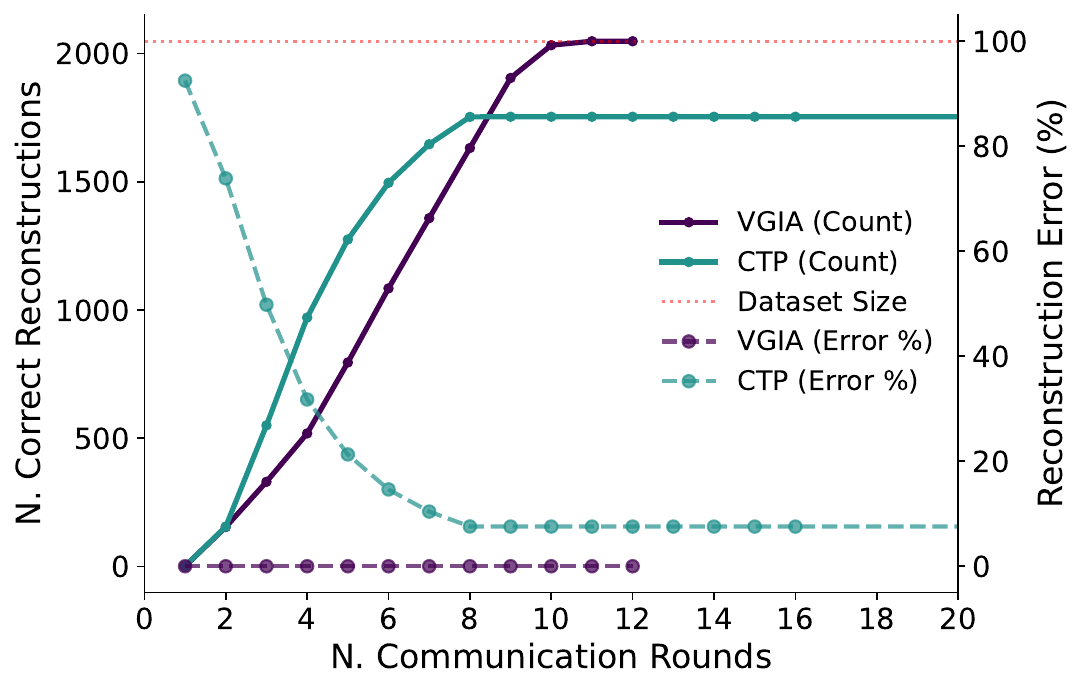}
    }\hspace{0.025\linewidth}
    \subfloat[VGIA V.S. CTP$_{<}$ \label{fig:income_efficiency}]{%
        \includegraphics[width=0.275\linewidth]{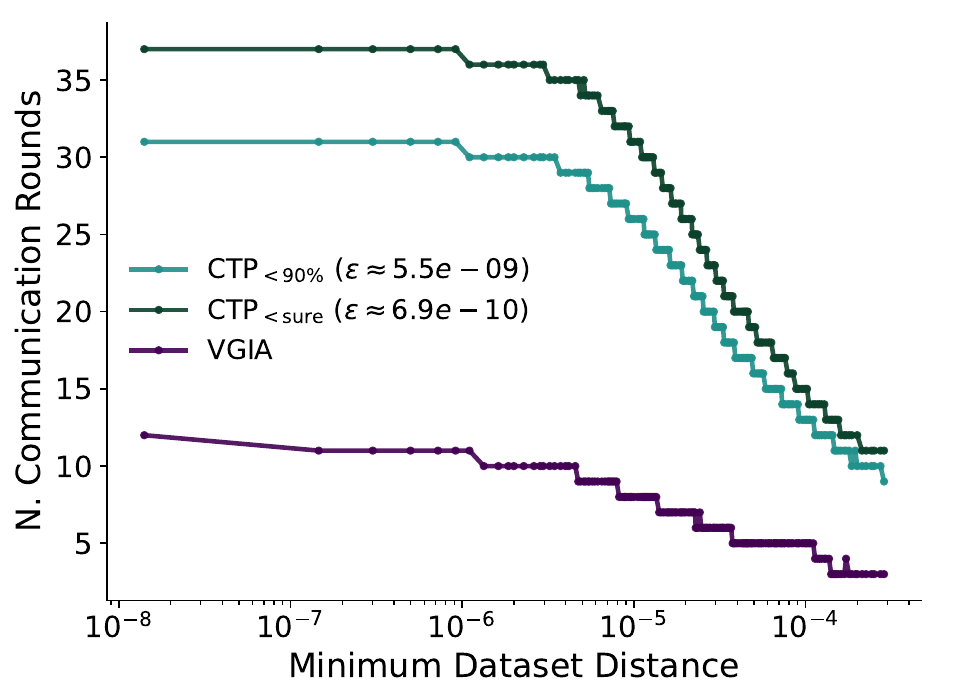}
    }


    \caption{
Attack performance of VGIA and the CTP baselines when the target client trains a three-layer fully connected neural network on a random subset of ACS Income dataset (size 2048) using full-batch FedSGD. (a,b) Number of correct samples and reconstruction error under varying attack-round budgets. (c) Number of attack rounds required for the adversary to identify all correct samples (i.e., achieve verifiability) under different target dataset with varying $\varepsilon_{\mathbf w}$. CTP$_{=}$ presents a baseline with knowledge of $\varepsilon_{\mathbf w} = $ (the minimum absolute distance between the projections of any two samples along direction $\mathbf{w}$). 
CTP$_{>}$ presents a optimistic setup on $\varepsilon$, s.t., $\varepsilon=10^{-4} > \varepsilon_{\mathbf w}$. 
CTP$_{<}$ presents a pessimistic setup on $\varepsilon$, s.t., $\varepsilon < \varepsilon_{\mathbf w}$.
  }
    \label{fig:main_fig}
\end{figure*}

\paragraph{Metrics}
In VGIA and CTP attacks, the adversary can obtain reconstructions that are either correct samples or linear combinations of true samples (spurious reconstructions).\footnote{VGIA may still produce occasional incorrect reconstructions due to numerical errors. For both attacks, we deem a recovered input correct if its $\ell_2$ distance to the corresponding ground-truth sample is below $10^{-9}$.
} We report first two ground-truth-based metrics: the number of correct samples and the proportion of spurious reconstructions (termed reconstruction error) under a given number of attack rounds (called attack budget). We also report the number of attack rounds required for the adversary to identify all correct samples in $\mathcal D$. Remind that, for CTP$_<$ and CTP$_=$, such identification is feasible, once all search slices reach the $\varepsilon$ distance. For VGIA, the adversary can identify the true sample once Prop.~\ref{prop:sample_isolation} holds.


\subsection{Experimental Results}\label{sec:exp_results}

Figure~\ref{fig:income_good_epsilon}  and~\ref{fig:income_bad_epsilon} report the performance of VGIA and CTP (under ground-truth-based metrics) on the target datasets drawn from the ACS Income dataset (size 2048) under various attack-round budgets.

In Fig.~\ref{fig:income_good_epsilon}, we consider the most favorable setting from the baseline, CTP$_=$, where $\varepsilon=\varepsilon_{\mathbf w}$. VGIA correctly recovers all 2048 samples within 11 attack rounds, whereas CTP requires $20$ rounds. Moreover, once the attack budget exceeds 11 rounds, VGIA can verify all recovered samples at round 12, eliminating the need for further attacks.  
Although CTP recovers more correct samples than VGIA for attack budgets between 2 and 8 rounds, the attacker cannot distinguish true samples from spurious reconstructions. For instance, after 4 attack rounds, CTP recovers nearly 1000 correct samples but incurs a reconstruction error of 30\%, corresponding to approximately 430 spurious samples. For  attack budget equal to~8, CTP recovers 1750 correct samples but also has 120 spurious samples. In contrast, VGIA produces no spurious reconstructions throughout this experiment.

In Fig.~\ref{fig:income_bad_epsilon}, we consider a baseline scenario, CTP$_>$, with $\varepsilon = 10^{-4} > \varepsilon_{\mathbf w}$, corresponding to an overestimation of the true threshold $\varepsilon_{\mathbf w}$. In this setting, the baseline fails to recover all samples. 
Moreover, even for large attack budget, approximately 8\% of the reconstructions are spurious and cannot be filtered out by the attacker. This phenomenon is more evident for KCH dataset where samples are closer to each other (see Fig.~\ref{fig:king_county_bad_epsilon} in Appendix), VGIA still recovers all samples within 11 rounds, whereas CTP$_>$ recovers at most 1000 samples and produces approximately 420 spurious reconstructions.

\paragraph{Verifiability study} In Fig.~\ref{fig:income_efficiency}, 
we study attacks' verifiability across different target client datasets sampled from the full dataset $\mathcal S$, for which the ground-truth
$\varepsilon_\mathbf w$ ranges from $\varepsilon_{\min}=6.9\times10^{-10}$ to $10^{-3}$.
 Note that CTP can successfully verify all the true samples only when $\varepsilon \leq \varepsilon_{\mathbf w}$. In practice, the attacker does not know $\varepsilon_{\mathbf w}$ which depends on the target dataset, but, in this experiment,  
we assume the CTP attacker knows the CDF of \(\varepsilon_{\mathbf w}\) and we consider two choices:
(i) \(\varepsilon=\varepsilon_{\min}\), which guarantees verifiability for any dataset drawn from $\mathcal S$ (CTP$_{<\text{sure}}$); and
(ii) $\varepsilon$ set to the $(0.1)$-quantile of the CDF, which yields a $90\%$ probability of verifiability (CTP$_{<90\%}$).
From the figure, we observe that both CTP$_{<\text{sure}}$ and CTP$_{<90\%}$ require at least twice the attack budget of VGIA to achieve verification across all sampled target datasets. The relative speedup of VGIA remains nearly constant across all values of $\varepsilon_{\mathbf w}$.
The same observation maintains in HARUS and KCH dataset (see Figs.~~\ref{fig:harus_efficiency} and~\ref{fig:kch_efficiency} in Appendix).

\section{Conclusion and Future Works}
In this work, we introduce a \emph{Verifiable Gradient Inversion Attack}, a novel analytical attack in federated learning that provides explicit certificates of correctness for reconstructed samples without restrictive assumptions.\\
While fully connected networks are widely used in tabular learning, extending the proposed attack to more expressive architectures remains an important direction for future work. For example, recent studies show that targeted modifications to Transformer models can induce behaviors enabling direct input recovery~\citep{fowl2023decepticons}. 
Finally, a systematic evaluation of attack performance under privacy-enhancing mechanisms remains largely unexplored and constitutes a compelling avenue for future research.

\section*{Impact Statement}
This work discusses an attack that highlights issues of user privacy in federated learning by reconstructing client data from model updates. While such attacks could be misused, the goal of this work is to raise awareness of these privacy threats, helping the community to understand the limitations of federated learning and to stipulate the development and usage of effective defense strategies to preserve user privacy.


\bibliography{ref}
\bibliographystyle{icml2026}

\newpage
\appendix
\onecolumn

\section{Proof of Proposition \ref{prop:sample_isolation}}
\label{app:proof_prop1}

\begin{proof}
Let $X = \{\x_1, \dots, \x_m\} \subset \mathbb{R}^d$ be the set of unknown private samples located within the two hyperplanes  $\mathbf w \x + \hat b_k^t=0$ and $\mathbf w \x + \hat b_{k+1}^t=0$. We assume these samples are linearly independent.

Recall from Equation \eqref{eq:hp_diff} that the server observes a linear combination of the isolated samples. At attack round $t-1$, the slice vector $\mathbf{s}_k^{t}$ is given by:
\begin{equation}
    \mathbf{s}_k^{t} = \sum_{j=1}^m \beta_j^{t} \x_j,
\end{equation}
where $\beta_j^{t} = \frac{1}{B}\frac{\partial \loss_j}{\partial z^{L}} \in \mathbb{R}$ are scalar coefficients derived from the prediction error and model weights at round $t-1$ \eqref{eq:rec_coeff}.
 
 At the next attack round, the attacker will test $q$ bias values in the interval $[\hat b_k^{t}, \hat b_{k+1}^t]$, that is 
$\hat b_{k}^{t+1} (= \hat b_{k}^{t}) < \hat b_{k+1}^{t+1} < \dots \hat b_{k+q-1}^{t+1} (=\hat b_{k+1}^{t})$, and compute the corresponding slice vectors $\slice_k^{t+1}, \slice_{k+1}^{t+1}, \dots \slice_{k+q-1}^{t+1}$.

Let $X_r \subset X$ denote the subset of samples falling into the $r$-th sub-interval $[\hat b_r^{t}, \hat b_{r+1}^t]$. Thus, we have:
\begin{equation}
    \slice^{t+1}_r = \sum_{\x_j \in X_r} \beta_j^{t+1} \x_j.
\end{equation}

Note that $\bigcup_{r=1}^n X_r = X$, and $X_r \cap X_q = \emptyset$ for $r \neq q$.

\paragraph{ Isolation $\implies$ Equal subspace dimension}
Assume the isolation is successful. This implies that each non-empty sub-slice $\slice^t_r$ isolates exactly one sample. Without loss of generality, assume $n=m$ and each $X_r = \{\x_r\}$. Then, according to Equation \eqref{eq:isolation} the new slice vectors become scaled versions of the original inputs:
\begin{equation}
    \slice^{t+1}_r = \beta_r^{t+1} \x_r.
\end{equation}
As a consequence, the subspace spanned by the new slice vectors is exactly the subspace spanned by the inputs, i.e.,:
\begin{equation}
\slice_k^t \in \mathrm{span}\left(\slice_k^{t+1}, \slice_{k+1}^{t+1}, \dots \slice_{k+q-1}^{t+1}\right).\nonumber 
\end{equation}

\paragraph{Equal subspace dimension $\implies$ Isolation with probability 1}
We proceed by contradiction. Assume that the inclusion of the previous round's slice vector $\slice_k^{t}$ does not increase the dimensionality of the subspace spanned by the current round's slices  $\slice_k^{t+1}, \slice_{k+1}^{t+1}, \dots \slice_{k+q-1}^{t+1}$, i.e.,
\begin{equation}
\slice_k^t \in \mathrm{span}\left(\slice_k^{t+1}, \slice_{k+1}^{t+1}, \dots \slice_{k+q-1}^{t+1}\right).\nonumber 
\end{equation}
but there exists at least a slice $\slice_r^{t+1}$ that contains $m \geq2$ samples (non-isolation). 
Since we assumed non-isolation, there exists at least a slice $\slice_r^{t+1}$ that is a collision of at least two samples $\x_1$ and $\x_2$. Without loss of generality, let $\slice_1^{t+1}$ contain samples $\x_1$ and $\x_2$. Then:

\begin{equation}
    \slice_1^{t+1} = \beta_1^{t+1}\x_1+\beta_2^{t+1}\x_2 
\end{equation}

For $\slice_k^{t}$ to belong to $\mathrm{span}(\slice_k^{t+1}, \slice_{k+1}^{t+1}, \dots \slice_{k+q-1}^{t+1})$ and not increase the subspace dimensionality, it must be expressible as a linear combination of the vectors $\slice_k^{t+1}, \slice_{k+1}^{t+1}, \dots \slice_{k+q-1}^{t+1}$. Focusing on the subspace spanned by $\x_1$ and $\x_2$, the condition requires that the ratio of coefficients at different rounds is the same for both  $\x_1$ and $\x_2$:
\begin{equation}
    \frac{\beta_1^{t+1}}{\beta_1^{t}}=\frac{\beta_2^{t+1}}{\beta_2^{t}},
\end{equation}
where, according to \eqref{eq:rec_coeff}, $\beta_j^t = \frac{1}{B}\frac{\partial \loss_j^t}{\partial z^{(L), t}_j}$ and $z^{(L), t}_j$ is the output prediction of model at around $t-1$ on input sample $\x_j$. 

As we assume $\W^L$ is drawn at each attack round from an absolutely continuous distribution with positive values, it follows that $z^{(L),t}_j$, being a linear combination of these weights with nonnegative coefficients that are not all zero, is itself absolutely continuously distributed.

Since we assume that  all level sets of $\frac{\partial \loss_j }{\partial z^L}$ are Lebesgue-null, then $\beta_j^t$ is also absolutely continuously distributed. Thus, it holds with probability 1 that


\begin{equation}
    \frac{\beta_1^{t+1}}{\beta_1^{t}}\neq\frac{\beta_2^{t+1}}{\beta_2^{t}}
\end{equation}
This inequality implies that $\slice_k^{t}$ contains a component linearly independent of $\slice_1^k$, forcing:
\begin{equation}
\slice_k^t \not \in \mathrm{span}\left(\slice_k^{t+1}, \slice_{k+1}^{t+1}, \dots \slice_{k+q-1}^{t+1}\right).\nonumber 
\end{equation}
This contradicts the initial assumption of rank equality. Thus, the subspace dimension equality condition implies strict isolation with probability 1.
\end{proof}

\section{Technical details}
\subsection{Extension to Multiclass Classification}\label{app:class_ext}

While the proposed attack is naturally formulated for tasks with scalar outputs, its core mechanism can be seamlessly extended to multiclass classification.
In standard classification settings, the model outputs a vector $\z^{(L)} \in \mathbb{R}^C$, and the loss is typically the cross-entropy between $\z^{(L)}$ and a one-hot target vector $\mathbf{y} \in \real^C$.
A naive application of the attack fails because the gradient signal backpropagated to the first layer is a linear combination of error terms from all $C$ classes, weighted by the connections of the second layer. Specifically, the gradient with respect to a bias $b_i^a$ becomes:

\begin{equation}
    \frac{\partial \mathcal{L}}{\partial b_i^a} = \sum_{c=1}^{C} \frac{\partial \mathcal{L}}{\partial z_c^{(L)}} g_{c,i} \mathbbm{1}_{z_i^{(1)}}.
\end{equation}

Since each neuron $i$ is connected to the output classes with different weights $\g_i$, the error term is no longer a sample-dependent scalar (as in the regression case, where it is $2(z^{(L)}-y)$), but a neuron-dependent mixture. This prevents the cancellation of the error term in Equation \eqref{eq:hp_diff}, which is necessary for sample isolation.

To overcome this limitation, inspired by \citet{fowl2022robbing}, the attacker modifies the classification weights $\W^{(L)}$ to force the classification network to mathematically behave like a regression model during the attack phase.
The server constructs the classification layer weight matrix $\W^{(L)}$ as a rank-1 matrix, defined by the outer product of a class projection vector $\mathbf{c} \in \mathbb{R}^C$ and a mixing vector $\mathbf{u} \in \mathbb{R}^N$:
\begin{equation}
    \W^{(L)} = \mathbf{c} \cdot \mathbf{u}^T.
\end{equation}
By substituting this structure into the gradient equation, the term $g_{k,i}$ becomes $c_k \cdot u_i$. Consequently, the gradient factorizes as:
\begin{equation}
    \frac{\partial \mathcal{L}}{\partial b_i^a} = \left( \sum_{k=1}^{C} \frac{\partial \mathcal{L}}{\partial z_k^{(L)}} c_k \right) u_i \mathbbm{1}_{z_i^{a} > 0}.
\end{equation}
The term in the parentheses is a scalar value: the projection of the multidimensional error vector onto the fixed direction~$\mathbf{c}$, which acts as a universal ``pseudo-residual'' common to all neurons.
The vector $\mathbf{u}$ then serves the same role as the weight vector in the scalar regression case.
Thus, by setting the malicious weights $\W^{(L)}$ to this rank-1 form, the server can apply the exact same reconstruction process derived in Section \ref{sec:attack_method}, using $u_i$ as the divisor in place of $g_i$ in Equations \eqref{eq:hp_diff} and \eqref{eq:rec_coeff}.
This transformation allows for the exact recovery of inputs in multiclass classification tasks without requiring any changes to the search or isolation algorithms. Additionally, target recovery can be performed by directly solving \eqref{eq:target_rec}, evaluating each candidate label $\hat{y}_k$ and selecting the one that minimizes the reconstruction error.

\subsection{Reconstruction of targets}
\label{app:reconstruction_target}
The final phase of the attack focuses on recovering the private target value ${y}_k$. While target inference has been widely studied in classification tasks \citep{yin_gi, idlg, kariyappa23a_cocktail, e2egi}, comparatively little work has addressed continuous regression targets, where the lack of discrete labels complicates both inference and verification. Given an exact reconstruction of the input ${\x}_k$, target recovery can be cast as univariate optimization problem. 
Specifically, the attacker seeks the target value $\hat{y}_k$ that minimizes the  distance  between the ground-truth partial derivative and  a virtual gradient locally computed by the server:
\begin{equation}\label{eq:target_rec}
    \hat{y}_k \in  \argmin_{y}\left(\frac{\partial \loss_{k}}{\partial {b}_i^a} - \frac{\partial {\loss}}{\partial {b^a_i}}({\x}_{k}, y)\right)^2.
\end{equation}
We note that $\frac{\partial \loss_{k}}{\partial {b}_i^a} = \frac{\partial \loss_{k}}{\partial z^L}\frac{\partial z^L}{\partial b_i^a}$. The attacker can compute $\frac{\partial z^L}{\partial b_i^a}$, because it knows the model parameters and the input $\x_k$, It can then obtain $\frac{\partial \loss_{k}}{\partial z^L}$ from~\eqref{eq:rec_coeff}. The second term is a function of $y$ the attacker can compute because it knows the model parameters and $\x_k$.
If the derivative of the loss with respect to the network output $\partial \loss/\partial z^L$ is monotone, as it is the case for the squared loss, then the problem in~\eqref{eq:target_rec} admits a unique solution.

\subsection{Algorithm Details}
\label{app:algo}
Here, we present our subroutine CheckIsolation in Algo.~\ref{algo:check_isolation}.

\begin{algorithm}[h]
    \caption{CheckIsolation}\label{algo:check_isolation}
    \textbf{Input}: set of search intervals $\mathcal{I}$, and set of strips $\mathcal{S} = \{\mathcal{S}_1, \dots \}$
    
    \begin{algorithmic}[1]
    \STATE$\mathcal{I_{\mathrm{new}} \leftarrow (); \quad \mathcal{R}\leftarrow ()}$
    \STATE $i \leftarrow 1; \quad  j \leftarrow 1$
    \STATE $M \leftarrow|\mathcal I|$ 
    \WHILE{$ j < M$}
    \STATE $(([l,u ], \slice'),\mathcal I) \leftarrow (\mathcal I_1, (\mathcal I_2, \dots, \mathcal I_{|\mathcal I|}))$
    \STATE $\mathcal{R}_{\textrm{temp}} \leftarrow (); \quad \mathcal{I}_{\textrm{temp}} \leftarrow (); \quad
    \mathcal{V}_{\textrm{temp}} \leftarrow ()$
        \STATE $(b',b'', \slice, \beta)\leftarrow \mathcal S_1$        \WHILE{$b'' \in [l,u] $}

            \IF{$\slice\neq \mathbf{0}$}\label{line:search_cond}
                \STATE $\mathcal{V}_{\textrm{temp}} \leftarrow (\mathcal{V}_{\textrm{temp}}, \slice)$
                \STATE $\mathcal{R}_{\textrm{temp}} \leftarrow (\mathcal{R}_{\textrm{temp}}, b', b'', \slice, \beta)$
                \STATE $\mathcal{I}_{\textrm{temp}} \leftarrow (\mathcal{I}_{\textrm{temp}},([b', b''], \slice))$
           \STATE $\mathcal S \leftarrow  (\mathcal S_2, \dots, \mathcal S_{|\mathcal S|})$ 
            \ENDIF 
                    \STATE $(b',b'', \slice, \beta)\leftarrow \mathcal S_1$
        \ENDWHILE
        \IF{$|\mathcal{V}_{\textrm{temp}}| > 0$}
            \STATE $r \leftarrow \dim\left(\mathrm{span}\left(\mathcal{V}
            _{\textrm{temp}}\right)\right)$
            \STATE $r_{\mathrm{aug}} \leftarrow \dim\left(\mathrm{span}\left(\mathcal{V}
            _{\textrm{temp}},\slice'\right)\right)$
            \IF{$r = r_{\mathrm{aug}}$}
               \STATE $\mathcal{R} \leftarrow \mathcal{R} \cup \mathcal{R}_{\textrm{temp}}$ 
            \ELSE
                \STATE $\mathcal{I_{\mathrm{new}}} \leftarrow (\mathcal{I_{\mathrm{new}}},\mathcal{I}_{\textrm{temp}})$
            \ENDIF
        \ENDIF
        \STATE  $j \leftarrow j+1$
    \ENDWHILE
    \STATE Return $(\mathcal I,\mathcal{I_{\mathrm{new}}}), \mathcal{R}$
    \end{algorithmic}
\end{algorithm}

\section{Additional Experimental Details}
\subsection{Datasets}\label{app:dataset}
For regression tasks, we evaluate our attack on the ACS Income dataset \citep{income} and the King County Housing dataset \citep{datasetlink}. The ACS Income dataset consists of census records from all 50 U.S. states and Puerto Rico and includes a diverse set of demographic attributes, with the goal of predicting individual income. In our experiments, we restrict the data to samples from the state of California. The King County Housing dataset contains information on residential property transactions in King County, Washington, and comprises 18 input features, with house price prediction as the target task. For both datasets, input features are rescaled to the $[0,1]$ interval, while target variables are standardized.

For classification tasks, we evaluate our approach on two datasets: the Human Activity Recognition Using Smartphones (HARUS) dataset \citep{harus} and CIFAR10 \citep{cifar10}. The HARUS dataset consists of motion signals collected from 30 subjects using embedded smartphone sensors and includes 561 input features and 6 activity labels corresponding to distinct physical activities. All input features are scaled to the $[-1,1]$ range. CIFAR10 is a benchmark image classification dataset comprising 60{,}000 color images of size $32 \times 32$ across 10 object classes; for this dataset, pixel values are rescaled to the $[0,1]$ interval.

\subsection{Attacks Configuration}\label{app:exp_setup}
To limit the impact of numerical errors, all the results for both VGIA and CTP are computed in double precision.
\paragraph{Income}
For experiments on the Income dataset, in VGIA the attacker initializes each row $\mathbf{W}_i^a$ of the attack layer with identical values drawn from a normal distribution, $\mathbf{W}^a \sim \mathcal{N}(0, 10^{-4})$. For the remaining layers, the attacker enforces positive weights and biases by sampling all parameters independently from the uniform distribution $U[0.01, 0.02]$. Beyond ensuring activation of these layers, no additional constraints are imposed on the parameter distributions.

For the CTP attack, the server samples each attack-layer weight row from a normal distribution, $\mathbf{W}_i^a \sim \mathcal{N}(0, 10^{-2})$, and initializes $\mathbf{W}^{(2)}$ and $\mathbf{W}^{(L)}$ sampling parameter values from the same distribution. Biases of the intermediate layers are independently drawn from $\mathcal{N}(0, 10^{-8})$, while the final-layer bias is set to $\b^{(L)} = 10^{30}$.

\paragraph{HARUS} 
For the HARUS dataset, the VGIA attacker constructs the attack layer by assigning all rows $\mathbf{W}_i^a$ the same initialization, with entries drawn from a Gaussian distribution $\mathcal{N}(0, 10^{-6})$. All remaining weights and bias terms are independently sampled from the uniform distribution $U[1, 2]$. 

For the CTP attack, we follow  the original implementation. The attack direction is sampled from $\mathcal{N}(0, 10^{-4})$, intermediate-layer weights are initialized using the same distribution as the input weights, and the final classification bias vector terms are fixed to $\b^{(L)} = 10^{30}$.

\paragraph{King County Housing}
In the King County Housing experiments, VGIA follows the same parameter initialization scheme adopted for the Income dataset. Specifically, the attack layer direction is sampled from a normal distribution, $\mathbf{W}_i^a \sim \mathcal{N}(0, 10^{-4})$, whereas all other weights and bias parameters are independently drawn from the uniform distribution $U[0.01, 0.02]$.

For the baseline configuration, all attack-layer neurons are initialized with identical weights $\mathbf{W}_i^a \sim \mathcal{N}(0, 10^{-2})$. The weights of the subsequent layer, $\mathbf{W}^{(2)}$, are sampled from the same distribution. Biases in the intermediate layers are drawn independently from $\mathcal{N}(0, 10^{-8})$. Finally, the output-layer bias is fixed to $\b^{(L)} = 10^{30}$, while the corresponding output weights are initialized from $\mathcal{N}(0, 10^{-8})$.

\paragraph{CIFAR10}
For experiments on CIFAR10, both attack methods adopt the same parameter initialization strategy used for the HARUS dataset. In the VGIA setting, all rows of the attack-layer weight matrix are initialized identically, with values drawn from $\mathcal{N}(0, 10^{-6})$, while all remaining weights and bias terms are constrained to be positive and independently sampled from the uniform distribution $U[1, 2]$. 

Also for the CTP attack, we follow the original configuration: the attack direction is sampled from $\mathcal{N}(0, 10^{-4})$, intermediate-layer weights are drawn from the same distribution as the input-layer weights, and the final-layer bias is fixed to $\b^{(L)} = 10^{30}$.

\section{Additional Experimental Results}\label{app:additional_results}
\subsection{Income}
\label{app:fedavg}

We conduct additional experiments on the ACS Income dataset to evaluate the accuracy of our attack under FedAvg. The inclusion of multiple local steps introduces local model drift, which complicates the reconstruction process as samples are no longer strictly isolated between parallel hyperplanes. Consequently, we devised a more robust mechanism for span verification to facilitate our search process. Specifically, the attacker employs Gram-Schmidt orthogonalization to verify the condition in Proposition~\ref{prop:sample_isolation}. By projecting out shared components, the presence of a sample between two hyperplanes yields a signal magnitude clearly distinguishable from the noise induced by local model updates. In our evaluation, we initialize all the neurons of the attack layer as identical values $\W^a_i \sim \mathcal{N}(0, 10^{-4})$, while all other model parameters are drawn from $U[10^{-3}, 2 \cdot 10^{-3}]$ . For the training procedure, we set the learning rate to
$10^{-4}$.

The results reported in Table~\ref{tab:minibatch} confirm that, as expected, increasing the amount of local computation degrades the attack’s accuracy, as isolating individual samples becomes more challenging. Nevertheless, the attack continues to achieve perfect reconstruction for at least 25\% of the dataset and exceeds 50\% recovery in most experimental settings. As discussed earlier, local updates may introduce false positives, since parameter drift–induced signals can be misinterpreted as the presence of samples within a given slice. Importantly, this phenomenon does not substantially impair the verification process: across most configurations, false positives account for less than 20\% of the recovered samples.

Finally, we observe that CTP is ineffective in regression scenarios under FedAvg. As noted by \citet[Appendix~F.1]{diana2025cutting}, their method depends on classification-specific weight manipulations to mitigate hyperplane drift; in the absence of these adjustments, local updates render their search strategy ineffective.
\begin{table}[ht]
    \centering
    \small 
    \caption{Evaluation of the VGIA attack under FedAvg framework. The client possesses $n=1024$ samples in its local dataset. The attack layer of the model has $N=1000$ neurons. All metrics are reported after 30 attack rounds. FP is short for false positives.}
    \label{tab:minibatch}
    
    \begin{tabular}{lccc}
        \toprule
        \textbf{Local Epochs} & \multicolumn{3}{c}{\textbf{Batch Size}} \\
        \cmidrule(lr){2-4}
         & \textbf{128} & \textbf{256} & \textbf{512} \\
         & \scriptsize (N. Correct Rec. / FP) & \scriptsize (N. Correct Rec. / FP) & \scriptsize (N. Correct Rec. / FP) \\
        \midrule
        1 & 560.0$\pm$39.6 / 17.5$\pm$2.1 & 802.5$\pm$44.5 / 47.0$\pm$1.4 & 803.7$\pm$20.8 / 182.0$\pm$114.8 \\
        \addlinespace[0.5em]
        2 & 418.5$\pm$38.9 / 19.5$\pm$2.1 & 646.5$\pm$24.7 / 57.0$\pm$2.8 & 863.3$\pm$31.2 / 116.0$\pm$9.8 \\
        \addlinespace[0.5em]
        3 & 324.0$\pm$5.7 / 28.5$\pm$2.1 & 520.5$\pm$0.7 / 67.5$\pm$7.8 & 761.3$\pm$3.2 / 136.0$\pm$7.0 \\
        \addlinespace[0.5em]
        5 & 258.5$\pm$37.5 / 26.5$\pm$0.7 & 403.0$\pm$19.8 / 67.5$\pm$9.2 & 547.0$\pm$39.9 / 137.7$\pm$17.0 \\
        \bottomrule
    \end{tabular}
\end{table}

\subsection{HARUS}
\begin{figure*}[tb]
    \centering 
    \subfloat[VGIA V.S. CTP$_{=}$\label{fig:harus_good_epsilon}]{%
        \includegraphics[width=0.3\linewidth]{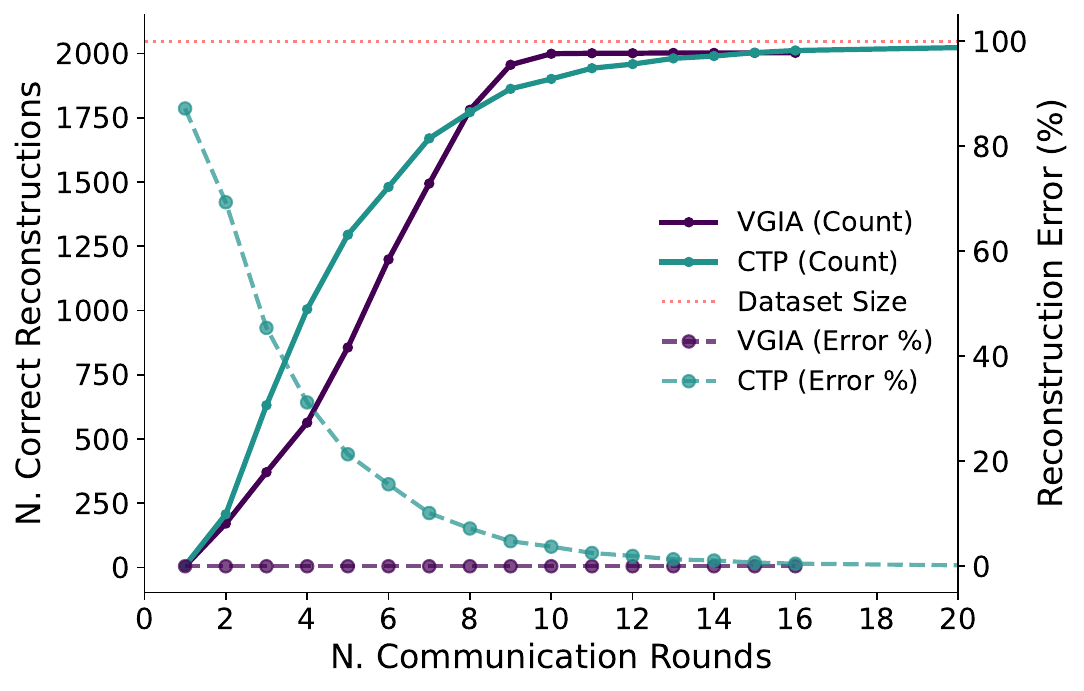}
    }\hspace{0.025\linewidth}
    \subfloat[VGIA V.S. CTP$_{>}$ \label{fig:harus_bad_epsilon}]{%
        \includegraphics[width=0.3\linewidth]{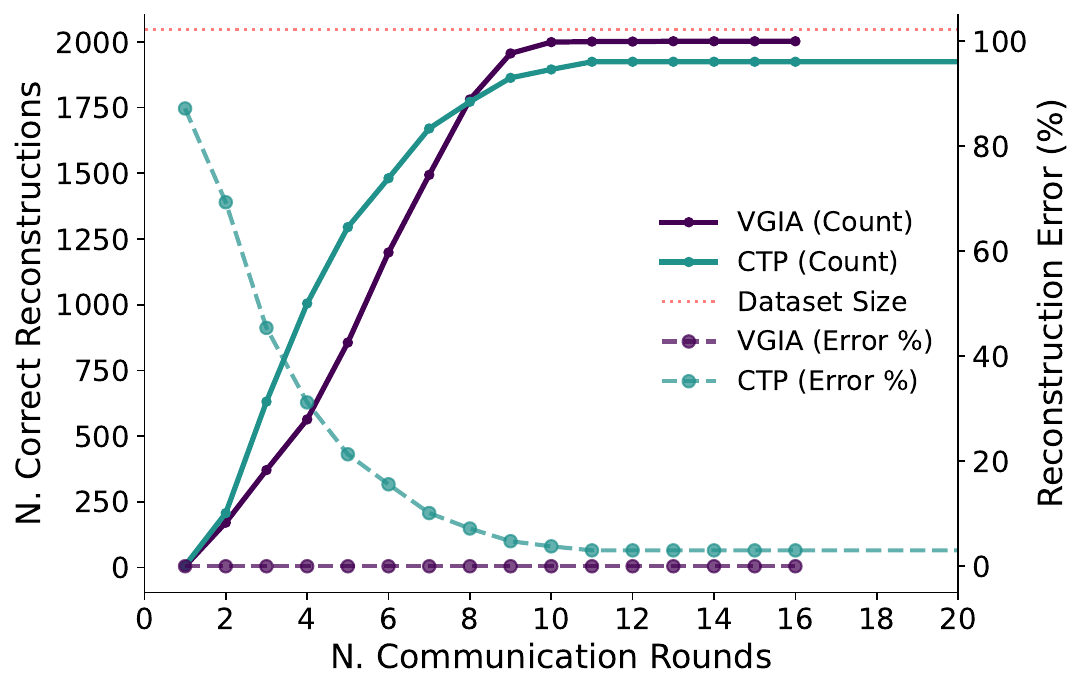}
    }\hspace{0.025\linewidth}
    \subfloat[VGIA V.S. CTP$_{<}$ \label{fig:harus_efficiency}]{%
        \includegraphics[width=0.275\linewidth]{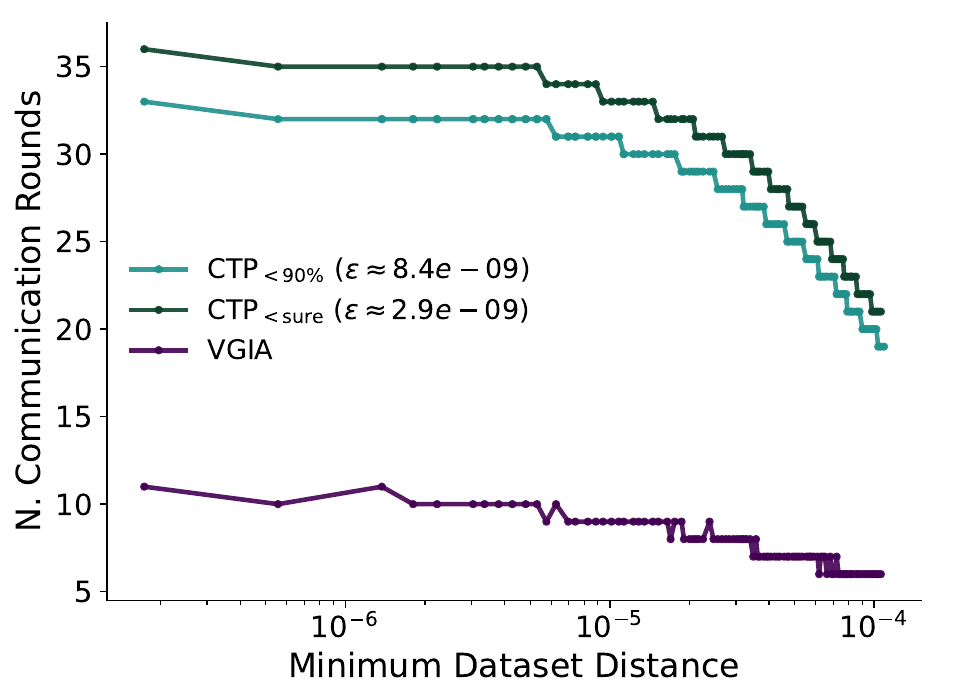}
    }
    \caption{
 (a,b) Number of correct samples and reconstruction error under varying attack-round budgets on HARUS dataset. (c) Number of attack rounds required for the adversary to identify all correct samples (i.e., achieve verifiability) under different target dataset with varying $\varepsilon_{\mathbf w}$. CTP$_{=}$ presents a baseline with knowledge of $\varepsilon_{\mathbf w}$ (the minimum absolute distance between the projections of any two samples along direction $\mathbf{w}$). 
CTP$_{>}$ presents a optimistic setup on $\varepsilon$, s.t., $\varepsilon = 10^{-4} > \varepsilon_{\mathbf w}$. 
CTP$_{<}$ presents a pessimistic setup on $\varepsilon$, s.t., $\varepsilon < \varepsilon_{\mathbf w}$. \label{fig:harus_res}
  }
    \end{figure*}
Figure~\ref{fig:harus_res} reports the performance of our VGIA attack on the HARUS dataset under the same experimental setting described in Section~\ref{sec:exp_results}, showing the adaptability of our attack to classification tasks. Consistent with the trends observed in the main paper, after an initial phase in which our attack is slowed down by the verification process, VGIA achieves faster convergence on HARUS. However, due to numerical precision errors, neither method is able to fully recover the dataset in this scenario.

Interestingly, when using the same $\varepsilon$ value as in Figure~\ref{fig:income_bad_epsilon}, the resulting performance is almost identical to that observed in Figure~\ref{fig:harus_bad_epsilon}. This highlights that tuning $\varepsilon$ is a dataset-dependent procedure that requires careful calibration. In practice, without additional prior knowledge, an attacker cannot reliably select an appropriate value, limiting the applicability of CTP in tabular domain.

Finally, Figure~\ref{fig:harus_efficiency} demonstrates the advantage of our adaptive search strategy, which allows our attack to save at least 10 communication rounds when $\varepsilon= \varepsilon_{\mathbf w}$. 

\subsection{King County Housing}

\begin{figure*}[h]
    \centering 
    \subfloat[VGIA V.S. CTP$_{=}$\label{fig:king_county_good_epsilon}]{%
        \includegraphics[width=0.3\linewidth]{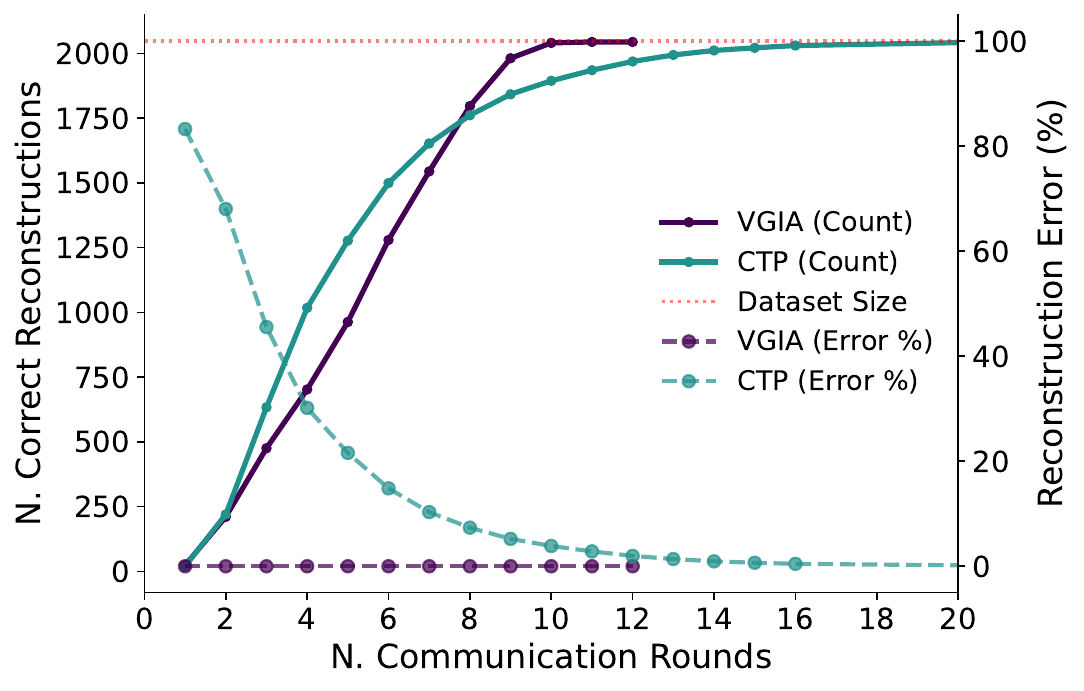}
    }\hspace{0.025\linewidth}
    \subfloat[VGIA V.S. CTP$_{>}$ \label{fig:king_county_bad_epsilon}]{%
        \includegraphics[width=0.3\linewidth]{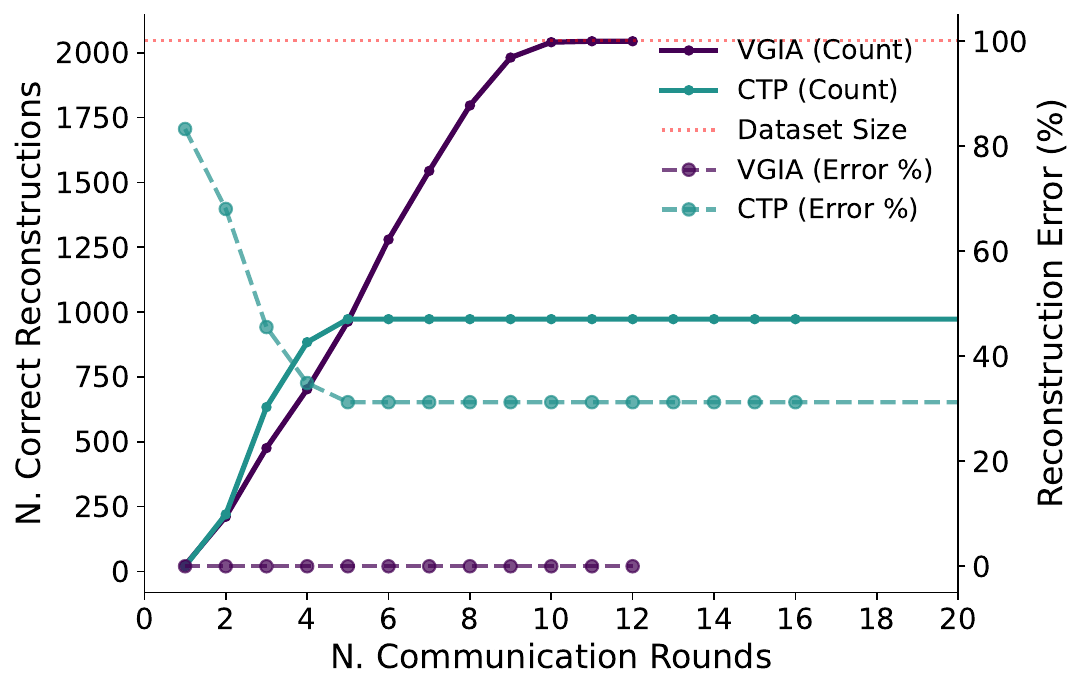}
    }\hspace{0.025\linewidth}
    \subfloat[VGIA V.S. CTP$_{<}$ \label{fig:kch_efficiency}]{%
        \includegraphics[width=0.275\linewidth]{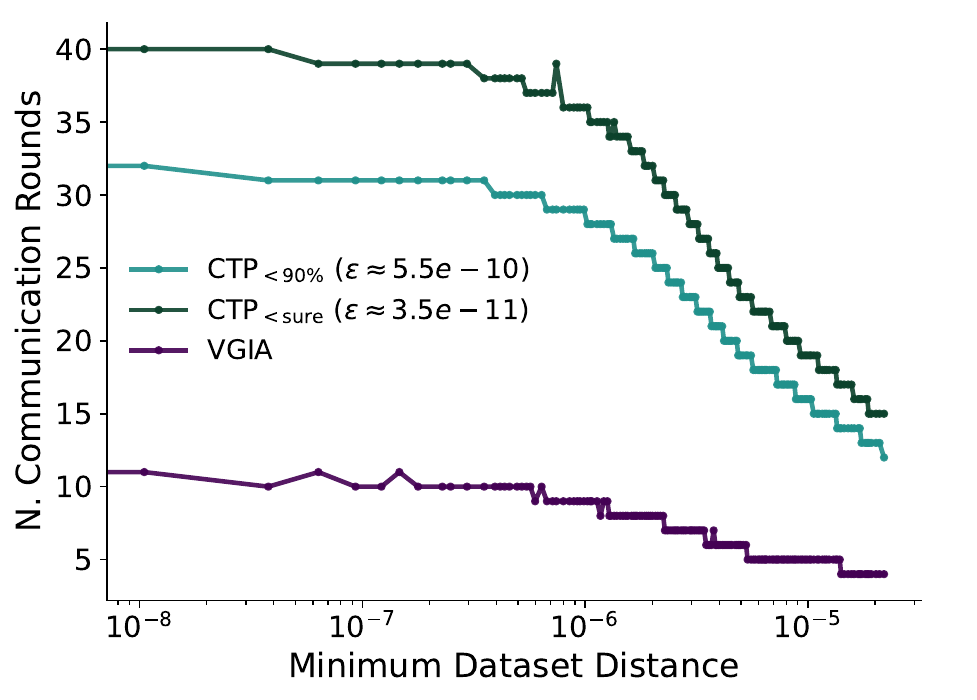}
    }
    \caption{
 (a,b) Number of correct samples and reconstruction error under varying attack-round budgets on King County Housing dataset. (c) Number of attack rounds required for the adversary to identify all correct samples (i.e., achieve verifiability) under different target dataset with varying $\varepsilon_{\mathbf w}$. CTP$_{=}$ presents a baseline with knowledge of $\varepsilon_{\mathbf w}$ (the minimum absolute distance between the projections of any two samples along direction $\mathbf{w}$). 
CTP$_{>}$ presents a optimistic setup on $\varepsilon$, s.t., $\varepsilon = 10^{-4} > \varepsilon_{\mathbf w}$. 
CTP$_{<}$ presents a pessimistic setup on $\varepsilon$, s.t., $\varepsilon < \varepsilon_{\mathbf w}$. \label{fig:king_county_res}
  }
    \end{figure*}
The results on the King County Housing dataset confirm the performance of our attack and its adaptability to different data distributions. As shown in Figure~\ref{fig:king_county_good_epsilon}, VGIA successfully recovers the entire dataset within 10 rounds and certifies the reconstruction quality by the subsequent round. In contrast, the CTP baseline is still attempting to isolate samples after 16 rounds.

Furthermore, Figure~\ref{fig:harus_bad_epsilon} highlights another aspect of the critical dependency of CTP on the choice of $\varepsilon$. Unlike results on HARUS, here we observe that an optimistic estimation of the separation distance compromises the reconstruction process, resulting in missing more than \%50 of the samples. Conversely, Figure~\ref{fig:kch_efficiency} illustrates the cost of the opposing condition: when $\varepsilon=\varepsilon_{\mathbf w}$, the baseline can require more than 40 communication rounds to verifiably isolate the inputs. In this regime, our adaptive approach saves almost 30 communication rounds compared to CTP.
\subsection{CIFAR10}
\begin{figure*}[h]
    \centering 
    \subfloat[VGIA V.S. CTP$_{=}$\label{fig:cifar_good_epsilon}]{%
        \includegraphics[width=0.3\linewidth]{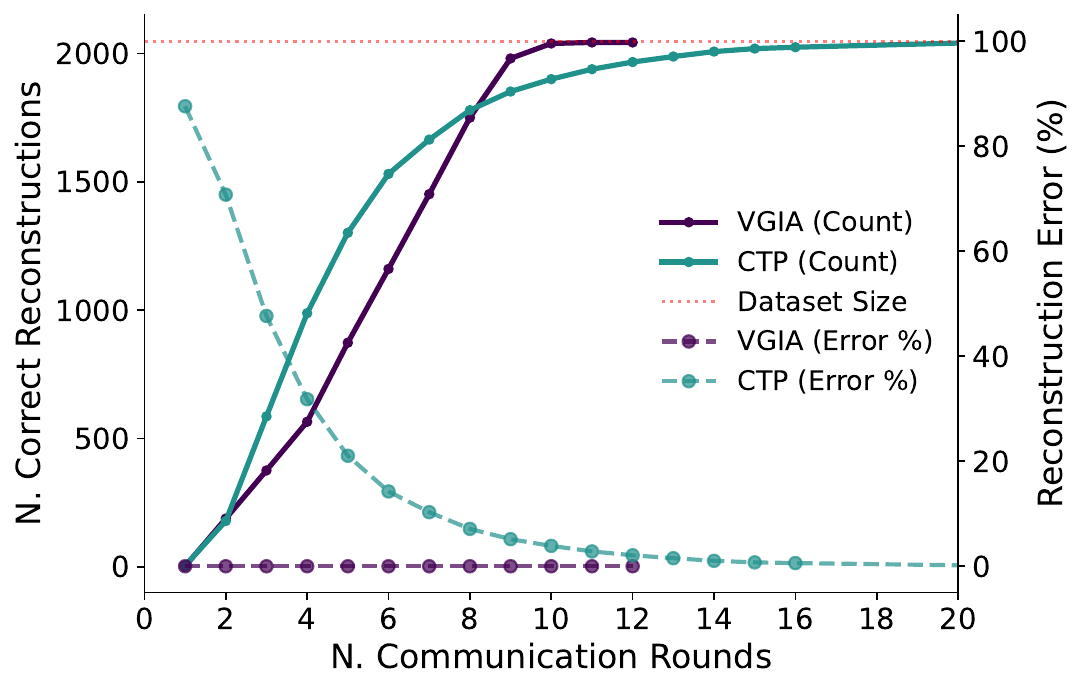}
    }\hspace{0.025\linewidth}
    \subfloat[VGIA V.S. CTP$_{>}$ \label{fig:cifar_bad_epsilon}]{%
        \includegraphics[width=0.3\linewidth]{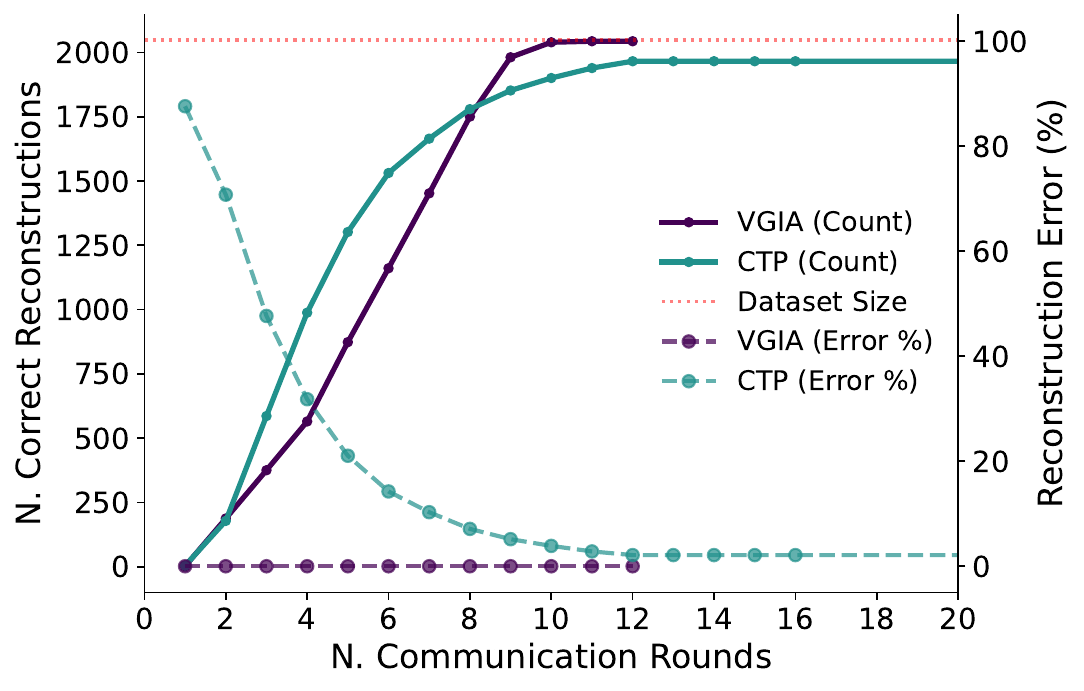}
    }\hspace{0.025\linewidth}
    \subfloat[VGIA V.S. CTP$_{<}$ \label{fig:cifar_efficiency}]{%
        \includegraphics[width=0.275\linewidth]{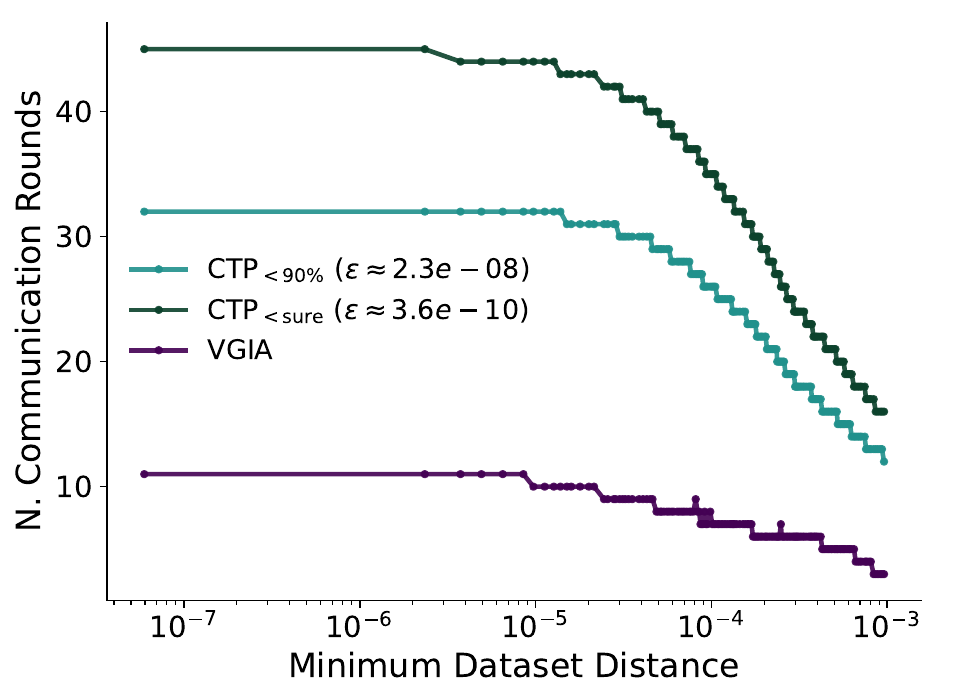}
    }
    \caption{
 (a,b) Number of correct samples and reconstruction error under varying attack-round budgets on CIFAR10 dataset. (c) Number of attack rounds required for the adversary to identify all correct samples (i.e., achieve verifiability) under different target dataset with varying $\varepsilon_{\mathbf w}$. CTP$_{=}$ presents a baseline with knowledge of $\varepsilon_{\mathbf w}$ (the minimum absolute distance between the projections of any two samples along direction $\mathbf{w}$). 
CTP$_{>}$ presents a optimistic setup on $\varepsilon$, s.t., $\varepsilon=10^{-4} > \varepsilon_{\mathbf w}$. 
CTP$_{<}$ presents a pessimistic setup on $\varepsilon$, s.t., $\varepsilon < \varepsilon_{\mathbf w}$. \label{fig:king_county_res}
  }
    \end{figure*}
Results on CIFAR10 remain consistent with those observed across all the other experimental settings, demonstrating that VGIA is not tied to a specific data modality and is therefore not restricted to tabular datasets. Despite not being explicitly designed for classification tasks, our attack completely recovers victim's dataset faster than CTP$_=$, which requires more than the 12 communication rounds needed by our method.

Figure~\ref{fig:cifar_bad_epsilon} further illustrates that VGIA converges more rapidly to the final solution even when the server selects an overly large value of $\varepsilon$, which prevents effective separation of the data points. For this dataset, setting $\varepsilon=10^{-4}$ does not substantially hinder the CTP attacker. By contrast, when the server is assumed to have access to the minimum admissible distance $\varepsilon_{\textrm{min}}$, CTP$_{\textrm{sure}}$ requires nearly four times as many rounds as VGIA, while CTP$_{<\textrm{90}}$ requires more than three times as many. Overall, these results highlight the efficiency of VGIA both in navigating the search space and in certifying the correctness of the reconstructed inputs.

\end{document}